\documentclass[11pt]{article}

\usepackage{fullpage}
\usepackage{amsmath}
\usepackage{amsfonts}
\usepackage{amsthm}
\usepackage{smile}
\usepackage{subcaption}

\usepackage{tablefootnote}

\usepackage[colorlinks,
            linkcolor=red,
            anchorcolor=blue,
            citecolor=blue
            ]{hyperref}
\usepackage{enumitem}
\usepackage{mathtools}

\usepackage{colortbl}
\definecolor{LightCyan}{rgb}{0.8, 0.9, 1}

\usepackage{amssymb}
\usepackage{pifont}
\newcommand{\cmark}{\ding{51}}%
\newcommand{\xmark}{\ding{55}}%

\ifdefined\final
\usepackage[disable]{todonotes}
\else
\usepackage[textsize=tiny]{todonotes}
\fi
\allowdisplaybreaks

\title{\huge Near-Optimal Regret for KL-Regularized Multi-Armed Bandits}

\author{
    Kaixuan Ji\thanks{Equal contribution} \thanks{Department of Computer Science, University of California, Los Angeles, CA 90095, USA; e-mail: {\tt kaixuanji@cs.ucla.edu}} 
    ~~
    Qingyue Zhao\footnotemark[1] \thanks{Department of Computer Science, University of California, Los Angeles, CA 90095, USA; e-mail: {\tt zhaoqy24@cs.ucla.edu}}
    ~~
    Heyang Zhao\footnotemark[1] \thanks{Department of Computer Science, University of California, Los Angeles, CA 90095, USA; e-mail: {\tt hyzhao@cs.ucla.edu}} 
    ~~
    Qiwei Di\thanks{Department of Computer Science, University of California, Los Angeles, CA 90095, USA; e-mail: {\tt qiwei2000@cs.ucla.edu}}
    ~~
    Quanquan Gu\thanks{Department of Computer Science, University of California, Los Angeles, CA 90095, USA; e-mail: {\tt qgu@cs.ucla.edu}}
}
\date{}

\newcommand{\piref}{\pi^{\mathsf{ref}}}
\newcommand{\kl}[2]{\ensuremath{{\mathsf{KL}}\left(#1\|#2\right)}}
\newcommand{\tv}[2]{\ensuremath{{\mathsf{TV}}\left(#1\|#2\right)}}

\newcommand{\fls}{{\bar{r}}} % reward estimated via least square
\newcommand{\fop}{{\hat{r}}} % reward estimated via least square + pessimism

\newcommand{\fgt}{{r}} % ground truth reward

\def \algcb {\text{KL-UCB}}

\newcommand{\pihat}{\hat{\pi}}
\newcommand{\subopt}{\mathrm{SubOpt}}

\newcommand{\unif}{\mathsf{Unif}}

\newcommand{\KL}{\mathsf{KL}}
\newcommand{\regret}{{\mathrm{Regret}}}

\begin{document}

\maketitle

\begin{abstract}

Recent studies have shown that reinforcement learning with KL-regularized objectives can enjoy \emph{faster} rates of convergence or \emph{logarithmic} regret, in contrast to the classical $\sqrt{T}$-type regret in the unregularized setting.
However, the statistical efficiency of online learning with respect to KL-regularized objectives remains far from completely characterized, even when specialized to multi-armed bandits (MABs).
We address this problem for MABs via a sharp analysis of \algcb{}~\citep{zhao2025logarithmic} using a novel peeling argument, which yields a $\tilde{O}(\eta K\log^2T)$ KL-regularized regret upper bound: the \emph{first} high-probability regret bound with linear dependence on $K$.
Here, $T$ is the time horizon, $K$ is the number of arms, $\eta^{-1}$ is the regularization intensity, and $\tilde{O}$ hides all logarithmic factors except those involving $\log T$.
The near-tightness of our analysis is certified by the \emph{first} non-constant lower bound $\Omega(\eta K \log T)$, which follows from subtle hard-instance constructions and a tailored decomposition of the Bayes prior.
Moreover, in the low-regularization regime (i.e., \emph{large} $\eta$), we show that the KL-regularized regret for MABs is $\eta$-independent and scales as $\tilde{\Theta}(\sqrt{KT})$.
Overall, our results provide a thorough understanding of KL-regularized MABs across all regimes of $\eta$ and yield nearly optimal bounds in terms of $K$, $\eta$, and $T$.
\end{abstract}

\section{Introduction}

Recently, many variants of the \emph{KL-regularized objective} $J(\pi) \coloneqq \EE_\pi r - \eta^{-1}\kl{\pi}{\piref}$ have become increasingly important in practice for bandits~\citep{rafailov2023direct,guo2025deepseek} and reinforcement learning (RL)~\citep{schulman2017proximal,ouyang2022training}, where $r$ is the mean reward function, $\piref$ is the reference policy, $\eta^{-1}$ is the regularization intensity, and $\mathsf{KL}$ is the reverse Kullback-Leibler divergence. For example, they have been instantiated as entropy regularization to strengthen the policy robustness~\citep{williams1992simple,ziebart2008maximum,levine2013guided,haarnoja2018soft}, and are widely employed in recommender systems~\citep{steck2018calibrated,geyik2019fairness} and large language model fine-tuning~\citep{ouyang2022training,rafailov2023direct,richemond2024offline,liu2024enhancing,guo2025deepseek}.
% and large language model finetuning~\citep{ouyang2022training,rafailov2023direct,richemond2024offline,liu2024enhancing,guo2025deepseek} to ensure training stability.

Given the prevalence of KL-regularized objectives, a growing body of work has been devoted to understanding the KL-regularized \emph{statistical} efficiency of decision making, where suboptimality is defined with respect to the regularized objective. \citet{xiong2024iterative,xie2025exploratory} demonstrate the rate of $\tilde{O}(\epsilon^{-2})$ for learning an $\epsilon$-optimal policy in contextual bandits and Markov decision processes. Starting from the pioneering \citet{tiapkin2023fast,zhao2025sharp}, previous works on this line (ignoring other factors)
% such as the number of actions)
either achieve an $\tilde{\Theta}(\epsilon^{-1})$ sample complexity~\citep{zhao2025sharp,zhao2026towards,foster2025good} or $\mathrm{polylog}(T)$ regret~\citep{zhao2025logarithmic,wu2025greedy} in various interaction protocols. In particular, \citet{tiapkin2023fast} obtained a fast-rate sample complexity in the pure exploration setting for both tabular and linear MDPs. 
\citet{zhao2025sharp} works in the hybrid offline setting under a strict uniform data coverage assumption.
% which was relaxed to various notions of single-policy concentrability in the pure offline setting~\citep{zhao2025sharpanalysisofflinepolicy} and in a hybrid setting with linear function approximation~\citep{foster2025good}. 
For online learning, \citet{zhao2025sharp} gives the first $\Omega(\eta \log(N_{\cR}))$ regret lower bound\footnote{See Remark~\ref{rmk:lowerbound-fast-comparison} for a detailed adaptation and discussion.} that does not scale
% neither 
with the time horizon $T$,
% nor the number of actions $K$ even if specialized to MABs, 
and \citet{zhao2025logarithmic} achieves the first logarithmic regret upper bound $\tilde O(\eta d_{\cR}  \log(N_{\cR}) \log T)$ under general function approximation, where $d_{\cR}$ is the eluder dimension and $\log(N_{\cR})$ is the metric entropy of the function class, following which \citet{wu2025greedy} design an algorithm free of bonus computation, which enjoys an $\tilde O\big(\eta \exp(2\eta) d_{\cR}\log(N_{\cR}T)\big)$ regret.
% in the low-regularization regime.
Therefore, all the previous foundational results leave the following problem open.
\begin{center}
    \emph{What is the exact regret of} online \emph{learning with} KL-regularized \emph{objectives?}
\end{center}
In this paper, we take the first step towards settling this question via a nearly sharp analysis for multi-armed bandits (MABs), a minimalist model of online learning. 
In particular, for KL-regularized MABs, we propose a variant of $\algcb$~\citep{zhao2025logarithmic} and provide regret upper bounds in both the high-regularization regime ($\eta$ small) and the low-regularization regime ($\eta$ large). We also construct two sets of hard instances that yield nearly matching regret lower bounds in both regimes, indicating that $\algcb$ is near-optimal.
Our two-fold contributions are as follows.

\begin{itemize}[leftmargin=*]
    \item We identify two complementary regimes with different regularization intensities, revealing the transition from $\sqrt{T}$-type regret to $\polylog(T)$-type regret as the regularization strength increases in KL-regularized MABs.
    
    \item For the high-regularization regime, our sharp analysis of \algcb{} yields a $\tilde{O}(\eta K \log^2 T)$ KL-regularized regret. 
    Correspondingly, we also provide a nearly matching $\Omega(\eta K \log T)$ lower bound, characterizing the regret behavior in this regime.
    % which is the \emph{first} upper bound of its kind that scales \emph{linearly} with the action space cardinality $K$. 
    % For the low regularization regime, we establish a $\tilde{O}(\sqrt{KT \log T})$ upper bound of the regularized regret for the same algorithm \algcb{}.
    % scales in a similar way to the unregularized setting.
    \item In the low-regularization regime, our analysis provides an $\tilde{O}(\sqrt{KT \log T})$ KL-regularized regret upper bound for the same algorithm \algcb{}, which nearly matches our established $\Omega(\sqrt{KT})$ lower bound, similar to the unregularized regret of MABs.
    % For both regimes, 
    % We provide an $\Omega(\eta K \log T)$ lower bound for high regularization regime and an $\Omega(\sqrt{KT})$ lower bound for low regularization. Together, these results match the upper bounds achieved by \algcb{} up to logarithmic factors and fully characterize the regret behavior of this problem.
    % nearly matching lower bounds are provided. In particular, our hardness analysis for the high regularization regime reveals intriguing and unique structures of the KL-regularized \emph{online} setting, and serves as the \emph{first} nontrivial (i.e., non-constant with respect to $T$) regret lower bound of its kind.
\end{itemize}
Our near-comprehensive understanding of KL-regularized MABs is visually demonstrated in \Cref{fig:regimes}. And relevant bounds on the statistical efficiency of KL-regularized decision making so far are summarized in \Cref{tab:kl-regularized} to ease comparison.

\newcolumntype{g}{>{\columncolor{LightCyan}}c}
\begin{table*}[ht!]
\centering
\caption{Comparison of regret or sample complexity upper and lower bounds for KL-regularized bandits. In this table, $T$ denotes total rounds of interactions, $\epsilon$ the target suboptimality gap, and $\eta$ the KL regularization coefficient. For the linear setting, $d$ denotes the dimension of the feature map. For general function approximation, $\cR$ is the function class, whose eluder dimension is $d_{\cR}$ and covering number is $N_{\cR}$, and $C_\mathrm{GL}^2$ is an instance-dependent constant that might be arbitrarily large. In the MAB setting, $K$ denotes the number of arms. $\tilde{O}(\cdot)$ hides logarithmic factors except $\log T$ and $\log (N_{\cR})$. A checkmark (\cmark) indicates that a matching (up to logarithmic factors) lower bound is known for the corresponding setting, while a cross (\xmark) indicates that no tight lower bound is currently available in its original setting and cannot match the lower bound when specialized to MAB.
}
\renewcommand{\arraystretch}{1}
\resizebox{\textwidth}{!}{
{
\begin{tabular}{lgggg}
\toprule
\rowcolor{white}Type &Algorithm  & Setting& Regret/Sample Complexity & Matching Lower Bound?   \\ 
\midrule
\rowcolor{white}  & Online Iterative GSHF \\
\rowcolor{white} & \small\citep{xiong2024iterative}& \multirow{-2}{*}{Preference w/ Linear Reward} &     \multirow{-2}{*}{$O(d^2/\epsilon^2)$}      &     \multirow{-2}{*}{\xmark}      \\ \rowcolor{white}
& TMPS \\\rowcolor{white}
& \small\citep{zhao2025sharp} & \multirow{-2}{*}{Data Coverage}&    \multirow{-2}{*}{$\tilde O\big((\eta ^2 C_\mathrm{GL}^2 + \eta/ \epsilon) \log (N_\cR)\big)$}   &    \multirow{-2}{*}{\xmark}       \\\rowcolor{white}
\multirow{-1}{*}{Upper Bound} & Greedy Sampling \\\rowcolor{white}
&\small\citep{wu2025greedy}& \multirow{-2}{*}{Preference w/ Eluder Dimension} &     \multirow{-2}{*}{$\tilde O\big(\eta \exp(2\eta) d_{\cR}\log(N_{\cR}T)\big)$}      &     \multirow{-2}{*}{\xmark}     \\\rowcolor{white}
& KL-UCB \\\rowcolor{white}
&\small\citep{zhao2025logarithmic} & \multirow{-2}{*}{Eluder Dimension}&   \multirow{-2}{*}{$\tilde O\big(\eta d_{\cR} \log T \log(N_{\cR})\big)$}     &        \multirow{-2}{*}{\xmark}   \\
 & KL-UCB & & &\\
&\small(This Work) & \multirow{-2}{*}{Multi-armed Bandits} &     \multirow{-2}{*}{$\tilde O(\eta K \log^2 T)$}      & \multirow{-2}{*}{\cmark}      \\
\midrule
\rowcolor{white} 
& \citet{zhao2025sharp} & Data Coverage&    $ \Omega\big(\eta \log(N_{\cR})/\epsilon\big)$   &    N/A     \rule{0pt}{2.5ex}  \\
\multirow{-3}{*}{Lower Bound}  & This Work & Multi-armed Bandits & $ \Omega(\eta K \log T)$  &    N/A  \rule{0pt}{3.5ex}\\[4pt] 
\bottomrule
\end{tabular}
}}
\label{tab:kl-regularized}
\end{table*}

\begin{figure}[t]
    \centering
    \includegraphics[width=0.618\linewidth]{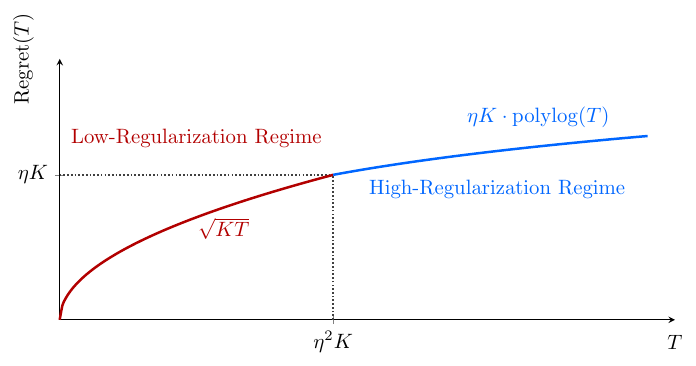}
    \caption{The near-comprehensive picture of KL-regularized MABs rendered in this paper. All logarithmic factors except $\log T$ are omitted to avoid clutter.}
    \label{fig:regimes}
\end{figure}

\paragraph{Notation.} The action set $\cA$ is assumed to be finite throughout the paper.
For nonnegative sequences $\{x_n\}$ and $\{y_n\}$, we write $x_n = O(y_n)$ if $\limsup_{n\to\infty}{x_n}/{y_n} < \infty$, $y_n = \Omega(x_n)$ if $x_n = O(y_n)$, and $y_n = \Theta(x_n)$, or alternatively, $y_n \sim x_n$, if $x_n = O(y_n)$ and $x_n = \Omega(y_n)$. We further employ $\tilde{O}(\cdot), \tilde{\Omega}(\cdot)$, and $\tilde{\Theta}(\cdot)$ to hide $\polylog$ factors. 
For finite $\cX$, we denote by $\mathcal{P}(\cX)$ the set of probability distributions on $\cX$, and by $\mathsf{Unif}(\cX)$ the uniform distribution on $\cX$.
% For probability measure $P$, we use supp$(P)$ to denote the support set of $P$.
For a pair of probability measures $\PP \ll \QQ$ on the same space, we use $\kl{\PP}{\QQ} \coloneqq \int \log({\mathrm{d} \PP}/{\mathrm{d} \QQ}) \ud \PP$ to denote their KL-divergence. For $p, q \in \RR$, we use $\cN(p, 1)$ to denote the normal distribution with expectation $p$ and unit variance, and we overload $\KL(p,q)$ to denote the KL-divergence between $\cN(p,1)$ and $\cN(q,1)$.
We denote $[N] \coloneqq \{1, \cdots, N\}$ for any positive integer $N$. Boldfaced lowercase letters, such as $\xb$, denote vectors, and their $i$-th entry is denoted by $x_i$. For finite $\cX$ and $\xb, \yb \in \cX^n$, we use $d_H(\xb, \yb) = \sum_{i=1}^n \ind(x_i \neq y_i)$ for their Hamming distance. $\forall a, b \in \RR$, $a \wedge b \coloneqq \min\{a, b\}$, $a \vee b \coloneqq \max\{a, b\}$, and $[a]_{[0, 1]} \coloneqq (a \vee 0) \wedge 1$.

\section{Related Work}
\paragraph{Optimism in Multi-armed Bandits.} We survey the paradigm of \emph{optimism in the face of uncertainty} for learning finite-armed bandits, which promotes exploration by favoring actions with high uncertainty.
The online interactive MAB setting originates from clinical scenarios~\citep{robbins1952some}, where minimizing the \emph{cumulative regret} is vital. \citet{lai1985asymptotically} initiates the algorithmic principle of optimism for learning MABs and gives the first asymptotic logarithmic regret lower bound.
% , whose leading factor is on the order of $\sum_{a: \Delta_a > 0} \Delta_a^{-1}$. Here the reward gap is $\Delta_a \coloneqq \max_{a^*}r(a^*)  - r(a)$. 
Under the simplification of Bernoulli noise, \citet{lai1987adaptive} proposes the algorithmic paradigm of \emph{upper confidence bound} (UCB), 
% following which many sharper and more general results are established 
which led to a sequence of works in the asymptotic regime~\citep{agrawal1995sample,burnetas1996optimal}. 
% In particular, \citet{burnetas1996optimal} gives the first asymptotically optimal UCB-type algorithm. 
To obtain a finite-time guarantee,
% guarantee a sublinear regret for any finite time horizon $T$, 
\citet{auer2002finite} proposes UCB1, which enjoys finite-time gap-dependent bounds.
% (and also gives the first finite-time analysis of $\epsilon$-greedy paradigm)
\citet{audibert2009minimax} achieves the first worst-case upper bound $O(\sqrt{KT})$ for MABs that is minimax optimal
% down to logarithmic factors 
via a UCB-type algorithm (MOSS). This influential UCB paradigm was later extended to be anytime optimal~\citep{degenne2016anytime} and both minimax and asymptotically optimal~\citep{lattimore2018refining}. Beyond classical MABs, the optimism principle has also been adopted in other online decision making problems including bandits with reward function approximation~\citep{abbasi2011improved,chu2011contextual,russo2013eluder}, structured bandits~\citep{kleinberg2008multi,chen2013combinatorial}, and Markov decision processes~\citep{zhang2024settling,zhou2022computationally}.

% Besides UCB, the algorithmic principle of \emph{randomized exploration} is efficient as well for balancing exploration and exploitation, which also enjoys similar gap-dependent~\citep{kaufmann2012thompson} and worst-case regret guarantees~\citep{agrawal2017near}.
% There is also a line of work on PAC bounds of best-arm identification (BAI) originates from \citet{even2002pac,mannor2004sample,even2006action,audibert2010best}, which coin the paradigm of \emph{successive elimination}, an algorithmic idea useful for both BAI and regret minimization. Finally, MABs with adversarial rewards~\citep{auer1995gambling,auer2002nonstochastic} or infinitely many arms~\citep{berry1997bandit,wang2008algorithms} are beyond our scope, and we refer the readers to \citet{lattimore2020bandit} for more detailed reviews.

\paragraph{RL with KL-Regularization.} Methods that use KL-regularized objectives have achieved strong empirical performance in (inverse) RL and its downstream applications~\citep{ziebart2008maximum,schulman2017proximal,ouyang2022training,guo2025deepseek}. Several lines of work aim to understand this paradigm. \citet{ahmed2019understanding,liu2021regularization} study the effect of entropy regularization on the stability of policy improvement in policy optimization, and related regret guarantees are analyzed in an online mirror descent framework by~\citet{cai2020provably,he2022near,ji2024horizonfree}. \citet{neu2017unified} places many KL-regularized algorithms in a unified optimization framework, and subsequent work analyzes the sample complexity of KL/entropy proximal methods in discounted MDPs with improved dependence on the effective horizon~\citep{geist2019theory,vieillard2020leverage,kozuno2022kl}. Nevertheless, because these works measure performance with the unregularized reward objective, the sample complexity for finding an $\epsilon$-optimal policy remains at the statistical lower bound $\Omega(\epsilon^{-2})$.

% fast rates major breakthrough \citep{zhao2025sharp,zhao2025sharpanalysisofflinepolicy,zhao2025logarithmic,foster2025good}\qingyue{to be extended}

When switching to performance with respect to the regularized objective, the fast rate $\tilde{O}(\epsilon^{-1})$ was first established by~\citet{tiapkin2023fast}, who derived a sample complexity of $\tilde{O}(H^5S^2K\eta/\epsilon)$ in the pure exploration setting, where $H$ is the horizon length, $S$ is the cardinality of the state space. Subsequently, \citet{zhao2025sharp} obtained a $\tilde{O}(\eta\epsilon^{-1}\log(N_{\cR}))$ sample complexity upper bound, albeit with an additional dependence on a notion of coverage that can be arbitrarily large. Moreover, \citet{zhao2025sharp} also provided an $\Omega(\eta \log (N_{\cR}) \epsilon^{-1})$ sample complexity lower bound, showing that the $\tilde{O}(\epsilon^{-1})$ rate is optimal. In the regret minimization setting, \citet{zhao2025logarithmic} first obtained an $\tilde{O}(\eta d_{\cR} \log N_{\cR} \log T)$ regret upper bound under reward function approximation. Later, \citet{wu2025greedy} obtained a $\tilde{O}\big(\eta \exp(2\eta) d_{\cR}\log(N_{\cR}T)\big)$ regret bound without constructing an exploration bonus. This kind of fast convergence against KL-regularized objectives has also been shown for pure offline~\citep{zhao2026towards,foster2025good}, competitive multi-agent~\citep{nayak2025achieving}, and privacy-constrained~\citep{wu2025offline,weng2025improved} settings. Nonetheless, no previous results match currently available worst-case lower bounds with respect to all problem parameters, such as $K$ and $\eta$.

\section{Problem Setup}\label{sec:cb:setup}

We denote a MAB with a KL-regularized objective by a tuple $(\cA, r, \eta, \piref, T)$, where $ K \coloneqq |\cA| < \infty$ is the number of actions, $r: \cA \to [0, 1]$ is the reward function unknown to the learner, $\eta >0$ is the ``inverse temperature'', $\piref \in \mathcal{P}(\cA)$ is a known reference policy, and $T \geq 1$ is the total number of interactions. At each round $t \in [T]$, the learner selects an action $a_t \in \cA$ according to a policy $\pi_t \in \mathcal{P}(\cA)$ and observes a noisy reward $r_t = r(a_t) + \varepsilon_t$, where $\varepsilon_t$ is $1$-sub-Gaussian~\citep[Definition~5.2]{lattimore2020bandit}. The learner's goal is to minimize the KL-regularized regret:
\begin{align*}
    \regret(T) = \sum_{t=1}^T \big[ J(\pi^*) - J(\pi_t) \big],
\end{align*}
where the objective $J(\pi)$ is defined as
\begin{align}
\label{eq:objective}
     J(\pi) = \EE_{a \sim \pi} \bigg[r(a) - \eta^{-1} \log \frac{\pi(a)}{\piref(a)} \bigg].
\end{align}
Equivalently, $J(\pi) = \EE_{a \sim \pi}[r(a)] - \eta^{-1}\KL(\pi\|\piref)$, i.e., the objective subtracts a KL penalty that discourages deviations from the reference policy $\piref$.
The regularization strength is controlled by $\eta$: smaller $\eta$ corresponds to stronger regularization.

% \color{red}

\begin{remark}
For large $\eta$, the effect of KL regularization is small and hence the objective (regret) is close to the objective (regret) without regularization. On the other hand, when $\eta$ is small, the objective forces the learned policy to stay close to $\piref$, resulting in a large difference from the unregularized objective and regret. Consequently, an algorithm achieving logarithmic KL-regularized regret might not achieve optimal ordinary cumulative regret.
\end{remark}

Under this objective, it is well known that the (unique) optimal policy $\pi^* \coloneqq \argmax_{\pi \in \mathcal{P}(\cA)} J(\pi)$ has the closed-form expression (see, e.g., \citealt[Proposition~7.16]{zhang2023mathematical})
\begin{align}\label{eq:opt-exp}
    \pi^*(\cdot) \propto \piref(\cdot)\exp\big(\eta \cdot r(\cdot)\big).
\end{align}
Moreover, for any reward function $r$, let $\pi_r^*$ denote the corresponding optimal policy.
For any policy $\pi \in \mathcal{P}(\cA)$, we define the suboptimality gap of $\pi$ (relative to $\pi_r^*$) by
\begin{align*}
    \subopt_r(\pi, \pi^*_r) &= \EE_{a \sim \pi^*_r} \bigg[r(a) - \eta^{-1} \log \frac{\pi^*_r(a)}{\piref(a)} \bigg] - \EE_{a \sim \pi} \bigg[r(a) - \eta^{-1} \log \frac{\pi(a)}{\piref(a)} \bigg].
\end{align*}

\begin{remark}
In the case of $\piref(a) = 0$ but $\pi(a) > 0$ for some $a\in \cA$, we have $\kl{\pi}{\piref} = +\infty$ and hence $J(\pi) = -\infty$. This case never arises when $\pi$ is given by $\pi(\cdot) \propto \piref(\cdot)\exp\big(\eta \cdot \hat r(\cdot)\big)$ provided that $\hat r$ is bounded on $\cA$.
\end{remark}

\section{Algorithm and Regret Analysis}\label{sec:algorithm}

\begin{algorithm*}[t]
\caption{KL-regularized Upper Confidence Bound Algorithm (\algcb)}\label{algorithm:kl-ucb}
\begin{algorithmic}[1]
    \REQUIRE Regularization $\eta$, reference policy $\piref$, total rounds of interaction $T$, number of actions $K$, error probability $\delta$.
    \FOR{$t = 0, ... , T-1$}
    
    \STATE Set $N_t(a) = \sum_{i=1}^t \ind \{a_i = a\}$ for all $a \in \cA$
    % \IF{$N_t(a)=0$}
    % \STATE Set the empirical reward $\fls_t(a) \leftarrow 0$, bonus $b_t(a) \leftarrow 1$
    % \ELSE
    \STATE Compute the empirical reward $\fls_t(a)$ and bonus $b_t(a)$ as 
    \begin{align*}
        \fls_t(a) \leftarrow \frac{1}{N_t(a) \vee 1} \sum_{i=1}^t r_i \ind \{a_i = a\}, \quad b_t(a) \leftarrow \sqrt{\frac{2\log(TK/\delta)}{N_t(a) \vee 1}}
    \end{align*}
    \STATE Set $\fop_t(a) \leftarrow [\fls_t(a) + b_t(a)]_{[0,1]}$
    % \ENDIF
    \STATE Compute $\pi_{t+1}(a) \propto \piref(a) \exp \big(\eta \cdot \fop_t(a)\big)$, play action $a_{t+1} \sim \pi_{t+1}$, and observe $r_{t+1}$
    \ENDFOR
\end{algorithmic}
\end{algorithm*}

In this section, we present a variant of $\algcb$~\citep{zhao2025logarithmic}, an algorithm for learning MABs with KL-regularization and its corresponding theoretical guarantees. 

\subsection{Algorithm Description}

% \kaixuan{Discuss the difference with~\citet{zhao2025logarithmic}}
We summarize $\algcb$ in Algorithm~\ref{algorithm:kl-ucb}, which follows a similar design to its original version in~\citet{zhao2025logarithmic} with general function approximation. In particular, for each round $t \in [T]$, the algorithm first counts the number of times each arm $a$ has been selected, denoted by $N_{t-1}(a)$. Then, the empirical reward is computed using the empirical mean. As in previous works on bandits~\citep{auer2002finite,zhao2025logarithmic}, $\algcb$ adopts the principle of optimism in the face of uncertainty~\citep{auer2002finite,abbasi2011improved}. Unlike~\citet{zhao2025logarithmic}, which built the bonus function using the uncertainty with respect to the reward function class, we adopt the following standard bonus for MABs
\begin{align*}
    b_t(a) = \sqrt{\frac{2\log(TK/\delta)}{N_t(a) \vee 1}}, \ \forall a\in \cA.
\end{align*}
The following lemma shows that the obtained reward $\fop = \fls + b$ is indeed an optimistic estimation of the true reward function $r$.
\begin{lemma}\label{lem:optimistic}
Given $\delta >0$, let $\cE(\delta)$ denote the event that our constructed optimistic reward function is indeed larger than true reward mean, i.e.,
\begin{align}
\notag
    \cE(\delta) \coloneqq \Big\{  \big| \fls_t(a) - \fgt(a) \big| \leq b_t(a), \ \forall (t,a) \in [T] \times \cA \Big\}.
\end{align}
Then the event $\cE(\delta)$ holds with probability at least $1-2\delta$.
\end{lemma}

After obtaining the optimistic reward estimation, we construct the policy $\pi_{t+1}$ for time step $t+1$ to be the optimal policy regarding $\fop_t$, according to which an action $a_{t+1}$ is sampled and we observe the reward $r_{t+1}$. 

\subsection{Theoretical Guarantee}

The regret upper bound of Algorithm~\ref{algorithm:kl-ucb} is given by the following theorem.

\begin{theorem}\label{thm:upperbound}
With probability at least $1-3\delta$, the cumulative regret of Algorithm \ref{algorithm:kl-ucb} admits the following upper bounds, depending on the regularization level.
\begin{itemize}[leftmargin=*]
    \item  For low regularization ($\eta \ge \sqrt{T/K}$), the regret can be upper bounded as
\begin{align*}
    \regret(T) = \tilde{O}\big(\sqrt{KT\log T}\big).
\end{align*}
\item  For high regularization ($\eta \le \sqrt{T/K}$), the regret can be upper bounded as
\begin{align*}
    \regret(T) =\tilde{O}\big(\eta K \log^2T\big),
\end{align*}
\end{itemize}
where $\tilde O$ hides logarithmic factors in $1/\delta$ and $K$.
\end{theorem}

\begin{remark}
Previously, \citet{zhao2025logarithmic} obtained an $O\big(\eta d(\cR, \lambda, T)\log(N_{\cR}T/\delta)\big)$ regret under general function approximation, where $\cR$ is the reward function class, $d(\cR, \lambda, T)$ is the eluder dimension with regularization parameter $\lambda > 0$~\citep[Definition~3.3]{zhao2025logarithmic} and $N_{\cR}$ is the covering number of $\cR$. When specializing to MABs, a standard elliptical potential argument~\citep[Section~3.1]{zhao2025logarithmic} with one-hot feature mapping shows that $d(\cR, \lambda, T) = O(K \log T)$~\citep[Section~D.1]{russo2013eluder}, and $\log N_{\cR} = O(K\log T)$. Thus, the worst-case regret upper bound in \citet{zhao2025logarithmic} reduces to $O(\eta K^2\log^2 T)$ in the multi-armed setting. Compared with their result, the $O(\eta K \log^2T)$ regret in Theorem~\ref{thm:upperbound} is strictly better.
\end{remark}

Theorem~\ref{thm:upperbound} establishes the regret upper bound of Algorithm~\ref{algorithm:kl-ucb} in two separate regimes. When $\eta \ge \sqrt{T/K}$, the regret scales with $\tilde{O}(\sqrt{KT})$. In contrast, when the regularization is high, i.e., $\eta \le \sqrt{T/K}$, Algorithm~\ref{algorithm:kl-ucb} enjoys a logarithmic regret $O(\eta K \log^2 T)$. These two regimes arise from the two-term structure of the KL-regularized objective~\eqref{eq:objective}. When $\eta$ is large, the effect of the regularization term becomes negligible, so the reward term dominates. In this case, the problem resembles a standard MAB problem and therefore recovers the $\tilde{O}(\sqrt{KT})$ rate. Otherwise, the KL regularization term dominates. It introduces sufficient curvature into the reward estimation error, thereby yielding logarithmic regret.
% \end{remark}

\subsection{Proof Sketch of \Cref{thm:upperbound}}

The proof operates on the high-probability event $\cE(\delta)$, where the empirical reward satisfies $|\fls_t(a) - \fgt(a)| \le b_t(a)$ for all $(t, a)$. In the low regularization regime, we follow the usual UCB-type analysis routine~\citep{lattimore2020bandit}. The intriguing regime of high regularization is analyzed as follows. The full version is deferred to \Cref{sec:missing-proof-of-upper-bounds}.

% \paragraph{Fast Rate ($O(\eta K \log^2 T)$).}
We begin with the KL-regularized regret decomposition (\Cref{lem:taylor-expansion-optimism}). Under $\cE(\delta)$, the regret is bounded by the cumulative expected squared error, which scales with the inverse visitation counts:
\begin{align}
    \regret(T) &\leq \eta \sum_{t=1}^T \EE_{a \sim \pi_t}\big[\big(\fop_{t-1}(a) - \fgt(a)\big)^2\big] \lesssim \eta {\sum_{t=1}^T \EE_{a \sim \pi_t} \bigg[ \frac{1}{N_{t-1}(a) \vee 1} \bigg]}.  \label{eq:sketch-decomp}
\end{align}
To bound this sum of expectations, we decompose $\eqref{eq:sketch-decomp}$ into an on-policy term ($I_1$) and a martingale difference term ($I_2$): $\eqref{eq:sketch-decomp} = \eta(I_1 + I_2)$, (we omit the $\cdot \vee 1$ in this sketch to avoid notation clutter) where
\begin{align*}
    I_1 &\coloneqq \sum_{t=1}^T {1} / {N_{t-1}(a_t)}, \\
    I_2 &\coloneqq \sum_{t=1}^T \bigg( \EE_{a \sim \pi_t} \big[ {1} / {N_{t-1}(a)} \big] - {1} / {N_{t-1}(a_t)} \bigg).
\end{align*}
% \begin{align*}
%     \sum_{t=1}^T \EE_{a_t \sim \pi_t} \bigg[ \frac{1}{N_{t-1}(a_t)} \bigg] = \underbrace{\sum_{t=1}^T \frac{1}{N_{t-1}(a_t)}}_{I_1: \text{Realized}} + \underbrace{\sum_{t=1}^T \bigg( \EE_{a_t \sim \pi_t} \bigg[ \frac{1}{N_{t-1}(a_t)} \bigg] - \frac{1}{N_{t-1}(a_t)} \bigg)}_{I_2: \text{Martingale Difference}}.
% \end{align*}
\paragraph{Bounding $I_1$ (Harmonic Sum).} The realized term is bounded deterministically via the properties of the harmonic series. A double counting with respect to arms yields
% By grouping terms by arm actions, we obtain:
\begin{align*}
    I_1 = \sum_{a \in \cA} \sum_{i=1}^{N_{T-1}(a)} \frac{1}{i} \approx \sum_{a \in \cA} \log \big(N_{T-1}(a)\big) \lesssim K \log T.
\end{align*}

\paragraph{Bounding $I_2$ (Peeling Technique).} The terms in $I_2$ form a martingale difference sequence (MDS). While it can be safely ignored if we only require a guarantee on the expected regret, obtaining a high-probability regret upper bound requires a concentration argument. A straightforward invocation of the Azuma-Hoeffding inequality yields an $\tilde{O}(\sqrt{T})$ bound. Although such a bound is acceptable in the unregularized setting, where the overall regret is also of order $\tilde{O}(\sqrt{T})$, it would dominate and thus destroy the desired $O(\log T)$ bound in the KL-regularized setting. Therefore, we need a more delicate approach. 

In particular, let $x_t$ be an element of this MDS:
\begin{align*}
    x_t = \EE_{a \sim \pi_t} \bigg[ \frac{1}{N_{t-1}(a)} \bigg] - \frac{1}{N_{t-1}(a_t)},
\end{align*}
and let $y_t \coloneqq 1 /{N_{t-1}(a_t)}$ be the on-policy term. Since $\sum_t y_t \lesssim K\log T$, the typical magnitude of $x_t$ is much smaller than its trivial upper bound $1$, suggesting that the cumulative conditional variance of $\{x_t\}$ should also be small. This observation motivates the application of Freedman's inequality
% We aim to apply Freedman's inequality
(Lemma \ref{lem:freedman-m}) for bounding $I_2 = \sum_t x_t$ via the sum of conditional variances. However, a direct application requires a sufficiently tight bound on the conditional variance; using only a crude upper bound would yield $O(\sqrt{T})$ regret bound, and thus fail to obtain the desired fast rate. To resolve this issue, we employ a novel \emph{peeling technique}. Specifically, we use the inequality $\text{Var}(x_t | \cF_{t-1}) \le \EE [y_t | \cF_{t-1}]$ and truncate $\sum_t \EE [y_t | \cF_{t-1}]$ at different levels $2^i$ for a more fine-grained upper bound of $\sum_t \text{Var}(x_t | \cF_{t-1})$.
For each $i >0 $, we define the following event
\begin{align*}
    \cE_i(t) = \bigg\{ \sum_{s=1}^t \EE [y_s|\cF_{s-1}] \le 2^i \bigg\}.
\end{align*}
Since $\sum_{t=1}^T \EE [y_t | \cF_{t-1}] \le T$ surely, we only need to consider $i \le \log_2 T$. It is easy to check that for any $i$, $\{x_t\ind(\cE_i(t))\}_{t=1}^T$ remains an MDS. Moreover, its conditional variance can be upper bounded by $2^i$, i.e.,
\begin{align*}
    \sum_{t=1}^T \text{Var}\big(x_t \ind[\cE_i(t)] \big| \cF_{t-1}\big) \leq 2^i.
\end{align*}
Now we apply Freedman's inequality to $\{x_t\ind(\cE_i(t))\}_{t=1}^T$. Taking a union bound over all $i\le \log_2 T $, we conclude that with high probability, the following inequality holds simultaneously for all $i \leq \log_2 T$, 
\begin{align}
\label{eq:help001}
    \sum_{t=1}^T x_t\ind[\cE_i(t)] \le \tilde O\big(2^{i/2}\big) + \text{minor terms}.
\end{align}
We now select an index $i$ such that $\ind[\cE_i(t)]=1, \forall t$. Since $\sum_{t=1}^T \EE [y_t | \cF_{t-1}] \le T$ and \eqref{eq:help001} holds for all $i \le \log_2 T$ simultaneously, we can choose $i$ such that $2^i \sim \sum_{t=1}^T \EE [y_t | \cF_{t-1}]$. Then, \eqref{eq:help001} gives
\begin{align*}
    I_2 \lesssim \sqrt{ \sum_{t=1}^T \EE_{a \sim \pi_t} \Big[ \frac{1}{N_{t-1}(a)} \Big]} + \text{minor terms}.
\end{align*}
Finally, recalling that $I_1 + I_2 = \sum_{t=1}^T \EE_{a \sim \pi_t}[ {1}/{N_{t-1}(a)}]$ and $I_1 \lesssim K \log T$, we have
\begin{align*}
    (I_1+I_2) \lesssim K \log T + \sqrt{I_1+I_2} + \text{minor terms}.
\end{align*}
We can conclude the proof using the resulting quadratic inequality $X \le A\sqrt{X} + B \implies X \lesssim A^2 + B$.
% Now we consider the $i$-th inequality where $i$ is the index such that $2^i \sim \sum_{t=1}^T \EE [y_t | \cF_{t-1}]$ and obtain
% \begin{align*}
%     I_2 \lesssim \sqrt{ \sum_{t=1}^T \EE_{a \sim \pi_t} \bigg[ \frac{1}{N_{t-1}(a)} \bigg] \cdot \log T } + \log T.
% \end{align*}
% Solving the resulting quadratic inequality $X \le A\sqrt{X} + B \implies X \lesssim A^2 + B$ allows us to bound the total expected sum by $O(K \log T)$, yielding the high probability logarithmic regret.
\section{Lower Bounds}\label{sec:lower-bound}

In this section, we present two minimax lower bounds to show that $\algcb$ is nearly minimax optimal. We first present the lower bound in the low-regularization regime, where $\eta \geq \sqrt{T/K}$.

\begin{theorem}[Low-regularization regime]\label{thm:lowerbound-slow}
Given any $K \geq 9$ and $\eta \ge \sqrt{T\log^2 K/K}$, for any algorithm, there exists a KL-regularized $K$-armed bandit on which the algorithm suffers from $\Omega(\sqrt{KT})$ regret.
\end{theorem}
A change-of-variable argument $\tilde{T} \leftarrow T / \log^2 K$ then yields the following corollary in the regime of $\eta \ge  \sqrt{T/K}$.
\begin{corollary}
Given any $K \geq 9$ and $\eta \ge \sqrt{T/K}$, for any algorithm, there exists a KL-regularized $K$-armed bandit on which the algorithm suffers from $\Omega(\sqrt{KT}\log^{-1}K)$ regret.

\end{corollary}

In the high-regularization regime where $\eta \leq \sqrt{T/K}$, the regret lower bound is characterized by the following theorem.

\begin{theorem}[High-regularization regime]\label{thm:lowerbound-fast}
Given any $K \geq 2$, $0 < \eta \le \sqrt{T/K}$, for any algorithm, there exists a KL-regularized $K$-armed bandit on which the regret of the algorithm is at least $\Omega \big(\eta K \log \big(T / (\eta^2K) \big) \big)$. 
\end{theorem}

\begin{remark}\label{rmk:lowerbound-fast-comparison}
Previously, an $\Omega(\eta \log N_{\cR}/\epsilon)$ lower bound was introduced in~\citet{zhao2025sharp} for a $2$-armed contextual bandit, which implies an $\Omega(\eta /\epsilon)$ sample complexity for MABs\footnote{The $\log N_{\cR}$ scaling entirely arises from the size of the context set and hence reduces to a constant in MABs.}. In contrast, Theorem~\ref{thm:lowerbound-fast} establishes an $\Omega(\eta K \log T)$ lower bound and implies\footnote{While in general a regret lower bound does not imply a sample complexity lower bound, the sample complexity lower bound here can be derived from the proof of Theorem~\ref{thm:lowerbound-fast}. We refer to Remark~\ref{rmk:sample-complexity-lower} for details. } an $\Omega(\eta K/\epsilon)$ sample complexity, strictly improving upon the previous result.
\end{remark}

% \begin{remark}
When $\eta \ge \sqrt{T/K}$, Theorem~\ref{thm:lowerbound-slow} shows that any algorithm must incur $\tilde{\Omega}(\sqrt{KT})$ regret. On the other hand, when $\eta \le \sqrt{T/K}$, Theorem~\ref{thm:lowerbound-fast} shows that any algorithm must incur $\Omega(\eta K\log T)$ regret. Compared with the upper bound in Theorem~\ref{thm:upperbound}, these lower bounds together show that $\algcb$ is \emph{minimax optimal} up to logarithmic factors, and the logarithmic dependence on $T$ in the high-regularization regime is \emph{inevitable}.
% \end{remark}

Remarkably, a $\log T$ gap remains between the upper bound in Theorem~\ref{thm:upperbound} and the lower bound in Theorem~\ref{thm:lowerbound-fast}. We conjecture that this gap primarily arises from the high-probability nature of Theorem~\ref{thm:upperbound}: the analysis requires a union bound over all $T$ rounds, which introduces an additional $\log T$ factor. Closing this gap, potentially through an expected-regret analysis of minimax-optimal bandit algorithms such as MOSS~\citep{audibert2009minimax} in the KL-regularized setting, remains an interesting direction for future work.
% We conjecture that this 
%  gap could be removed through an analysis of the expected regret. This potentially involves adapting some minimax optimal bandit algorithms like MOSS in [1] to KL-regularized setting. We will incorporate this discussion in the revised version and leave a complete solution to this issue as future work.

\section{Proof Overview of Hardness Results}\label{sec:overview}

In this section, we provide an overview of the proofs in Section~\ref{sec:lower-bound}. 
We first discuss the proof of Theorem~\ref{thm:lowerbound-slow}, which corresponds to the low-regularization regime and is more similar to unregularized MABs. Accordingly, following previous works \citep{lattimore2020bandit}, we construct a hard instance class consisting of $K$ hard-to-distinguish instances. However, this is not sufficient for proving the lower bound in the high-regularization regime (Theorem~\ref{thm:lowerbound-fast}). We first explain why the classical construction fails, and then propose a new approach based on a more sophisticated family of instances.
Throughout this section, we assume that the reward noise is independent Gaussian with variance $1$ unless otherwise specified.

\subsection{Proof Overview of Theorem~\ref{thm:lowerbound-slow}}\label{sec:overview-slow}

In this theorem, we consider the low-regularization regime, where $\eta \gtrsim \sqrt{T \log^2 K/K}$. In this regime, the effect of regularization is negligible, thus the regularized problem can be viewed as a small perturbation of the unregularized bandit. Accordingly, we construct hard instances by adapting the standard unregularized bandit lower-bound construction (see, e.g., \citet{lattimore2020bandit}). Specifically, we fix $\eta >0$, $\cA = [K]$, $\piref = \unif(\cA)$. We construct the hard-to-distinguish instance set as follows:

Fix a constant $\Delta > 0$ to be specified later. For the first instance, we define the reward function $r_1$ by setting $r_1(1) = \Delta$ and $r_1(i) = 0$ for all $i \ge 2$. For the remaining instances, for each $k \in \{2, \dots, K\}$, we define $r_k$ by setting $r_k(i) = r_1(i)$ for all $i \neq k$ and $r_k(k) = 2\Delta$.

% and algorithm \textsf{Alg}. For the first instance, we set $r_1$ to be $r_1(1) = \Delta$ and $r_1(i) = 0$ for all $i \geq 2$. 
For any algorithm \textsf{Alg}, let $N_T(i)$ be the number of times arm $i$ is pulled in the first $T$ steps. By the pigeonhole principle, there exists $k \geq 2$ such that 
\begin{align*}
    \EE_{r_1, \textsf{Alg}}[N_T(k)] \leq \frac{T}{K-1},
\end{align*}
where the expectation is taken over the distribution jointly given by instance $r_1$ and \textsf{Alg}. 
% Now we construct the second instance $r_2$ as $r_2(i) = r_1(i)$ for all $i \neq k$ and $r_2(k) = 2\Delta$. 
Now we consider the KL-divergence between the trajectory distributions induced by instances $r_1$ and $r_k$. 
By the chain rule of KL-divergence and $\KL(0,2\Delta) = 2\Delta^2$ (\Cref{fact:kl-gaussian}), $\kl{\PP_1}{\PP_k}$ can be bounded by
\begin{align}
     \sum_{i=1}^K \EE_1[N_T(i)]\KL(r_1(i), r_k(i)) \leq \frac{2T\Delta^2}{K-1}, \label{eq:overview-slow-kl}
\end{align}
where we adopt the shorthand $\PP_i:=\PP_{r_i, \textsf{Alg}}$ to denote the probability distributions over trajectories induced by the interaction between algorithm \textsf{Alg} and instances $r_i$ for $i \in [K]$. When picking $\Delta \sim \sqrt{K/T}$, we have $\kl{\PP_1}{\PP_k} = O(1)$, indicating that under algorithm \textsf{Alg}, it is hard to distinguish $r_k$ from $r_1$\footnote{In general, the KL-divergence $\kl{\PP}{\QQ}$ between two distributions $\PP$ and $\QQ$ being a constant indicates that $\PP$ and $\QQ$ cannot be reliably distinguished.}.

Now we compute the cost of misidentifying the underlying reward function. For $i\in\{1,k\}$, let $\pi^*_i$ be the optimal policy corresponding to $r_i$, as defined in \eqref{eq:opt-exp}. We define the suboptimality gap between $\pi^*_i$ and any policy $\pi \in \mathcal{P}(\cA)$ under instance $i$ as $\subopt_i(\pi) = \subopt_{r_i}(\pi, \pi^*_i)$. A direct computation yields that 
\begin{align}
    & \subopt_1(\pi) + \subopt_k(\pi) \gtrsim \frac{1}{\eta} \log \frac{(\mathrm{e}^{\eta\Delta} + K-1)(\mathrm{e}^{2\eta\Delta} + K-2)}{(2\mathrm{e}^{\eta\Delta} + K-2)^2}. \label{eq:overview-slow-subopt-raw}
\end{align}
As $\Delta \sim \sqrt{K/T}$, our assumption on $\eta$ indicates $\eta\Delta \gtrsim \log K$. Then, all the $\Theta(K)$ terms in the denominator of~\eqref{eq:overview-slow-subopt-raw} can be replaced by $\exp(\eta\Delta)$, which yields
\begin{align}
     \subopt_1(\pi) + \subopt_k(\pi) &\gtrsim \frac{1}{\eta} \log \frac{\mathrm{e}^{\eta\Delta} \cdot \mathrm{e}^{2\eta\Delta}}{(3\mathrm{e}^{\eta\Delta})^2}  \gtrsim \Delta. \label{eq:overview-slow-subopt-final}
\end{align}
Consequently, \eqref{eq:overview-slow-subopt-final} shows that the algorithm incurs a per-step cost of $\Omega(\Delta)$ if it cannot distinguish $r_k$ from $r_1$ and therefore suffers $\Omega(T\Delta)$ regret over $T$ rounds. Combining~\eqref{eq:overview-slow-kl} and~\eqref{eq:overview-slow-subopt-final} with Le Cam's method (Lemma~\ref{lem:two-point}), we conclude that \textsf{Alg} suffers from  $\Omega(T\Delta) = \Omega(\sqrt{KT})$ regret, which finishes the proof.

\subsection{Proof Overview of Theorem~\ref{thm:lowerbound-fast}}
Although the hard instance constructed in Section~\ref{sec:overview-slow} is standard for MABs, it does not apply to the high-regularization (fast-rate) case. In this section, we first explain why this construction fails in that regime. To overcome the difficulty, we then introduce new proof techniques to derive a sharper fast-rate lower bound.
\paragraph{Failure of Instances in Section~\ref{sec:overview-slow}.}
At a high level, the lower bound proof in Section~\ref{sec:overview-slow} relies on two key steps: constructing a set of statistically indistinguishable instances by setting $\Delta = \sqrt{K/T}$~\eqref{eq:overview-slow-kl}, and demonstrating that the suboptimality is sufficiently large on at least one of these instances \eqref{eq:overview-slow-subopt-final}. 
% \kaixuan{Seems duplicate with the beginning of Section~\ref{sec:overview}}

We first demonstrate that, in the regime of high regularization, the curvature of the regularizer plays a crucial role, resulting in an $\Omega(\eta\Delta^2)$ rather than $\Omega(\Delta)$ suboptimality gap. Specifically, when $\eta$ is small, we can redo~\eqref{eq:overview-slow-subopt-raw} as follows:
\begin{align*}
    \text{\eqref{eq:overview-slow-subopt-raw}} & = \frac{1}{\eta}\log \bigg(1 + \frac{M}{(M+\exp(\eta\Delta))^2}\big(\exp(\eta\Delta)-1\big)^2\bigg),
\end{align*}
where $M = K-2 + \exp(\eta\Delta)$. When $\eta\Delta$ is small, $K-2$ dominates $\exp(\eta\Delta)$ and makes $M = \Omega(K)$. Now, applying a basic inequality regarding $\eta \Delta$ results in 
\begin{align}
    \text{\eqref{eq:overview-slow-subopt-raw}} \sim \frac{1}{\eta}\log \bigg(1 + \frac{\eta^2\Delta^2}{K}\bigg) \sim \frac{\eta\Delta^2}{K}. \label{eq:overview-fast-subopt-failure}
\end{align}
Ignoring the dependence on $K$, we see that the suboptimality gap is of order $\eta\Delta^2$. Moreover, using the choice of $\Delta\sim\sqrt{K/T}$, we obtain an $\Omega(\eta)$ regret bound, whose dependency on $K$ is loose compared with Theorem~\ref{thm:lowerbound-fast}.

The gap with respect to $K$ is primarily due to the fact that strong regularization toward $\piref = \unif(\cA)$ forces any near-optimal policy to remain close to the uniform policy. Consequently, the policy assigns only $O(1/K)$ probability mass to the specific arms $\{1, k\}$ where the instances differ. As a result, the cost of making an error in distinguishing $r_k$ from $r_1$ is diluted by a factor of $K$, i.e., from $\Omega(\eta\Delta^2)$ to $\Omega(\eta\Delta^2/K)$. Therefore, to manifest the $\Omega(K)$ dependency in Theorem~\ref{thm:lowerbound-fast}, we need a more sophisticated set of instances.
\begin{remark}
    Since the two-point-type constructions in the proofs of lower bounds in previous works on KL-regularized decision making~\citep{zhao2025sharp,zhao2026towards} are in spirit similar to the construction for \Cref{thm:lowerbound-slow} \emph{when specialized to MABs}, the reasoning above also implies that it is not promising to directly adapt their constructions to the online setting to show the correct scaling with respect to $K$.
\end{remark}

\paragraph{Instance Design.} 
% As previously discussed, instances that differ on only a single arm do not yield a sharp lower bound in the high-regularization regime.
% , as the regret incurred by misidentifying the optimal arm is not sufficiently large in this setting. 
To overcome the issue above, we instead consider a class of instances in which $\Omega(K)$ arms might have different rewards and thus require estimation.
% The set of instances discussed in Section~\ref{sec:overview-slow} fails because the reward only varies on one arm, making the cost of misclassifications too small. To overcome this issue, we make rewards on $\Omega(K)$ arm not determined and to be estimated. 
In particular, let $K$ be even and $A \coloneqq K/2$. We fix $\eta > 0$ and keep $\cA = [K]$ and $\piref = \unif(\cA)$. Let $\cV = \{\pm 1\}^A$ and we consider the rewards parameterized by $\bmu \in \cV$ such that $r_{\bmu}$ is given as follows:
\begin{align*}
    r_{\bmu} (i) & = \frac{1}{2} + \mu_i \Delta, \ \forall i \in [A]; \\
    r_{\bmu}(i) & =  \frac{1}{2}, \ \forall i \in [A+1, 2A],
\end{align*}
where $\Delta > 0$ is a parameter to be specified. Upon this set of instances, to distinguish one of the rewards $r_{\bmu}$ from all other rewards in $\{r_{\bnu}\}_{\bnu \in \cV}$, the learner has to determine all the $A = \Omega(K)$ entries of $\bmu$. 

\paragraph{Suboptimality Gap Computation.}

Our next step is to demonstrate that the regret accumulates across all arms where the learner fails to distinguish whether the mean reward is $1/2 + \Delta$ or $1/2 - \Delta$. Intuitively, for any $\bmu \in \cV$ and $i \in [K]$, $r_{\bmu}(i)$ is very close to $1/2$ and therefore all near-optimal policies put $\Theta(1/K)$ probability mass on each arm. Hence, similar to the argument of~\eqref{eq:overview-fast-subopt-failure}, once the learner makes an error in estimating some $r(k)$, the cost of this error is always $\Omega(\eta\Delta^2/K)$ regardless of the estimation on the other arms. Therefore, the cost accumulates and results in $\Omega(m\eta\Delta^2/K)$ total cost if the learner makes $m$ mistakes.

In particular, let $\bmu_1,\bmu_2 \in \cV$ be two instances such that $d_H(\bmu_1, \bmu_2) = m$. From now on, for $i=1,2$, let $r_i = r_{\bmu_i}$, $\pi^*_i$ be the optimal policy corresponding to $r_i$ and $\subopt_i(\pi) = \subopt_{r_i}(\pi, \pi^*_i)$ be the suboptimality gap between $\pi^*_i$ and any policy $\pi \in \mathcal{P}(\cA)$. A direct computation yields that 
\begin{align}
    & \subopt_1(\pi) + \subopt_2(\pi)  \gtrsim \frac{1}{\eta} \log \Bigg(1 + \frac{2Km}{\big(K\exp(\eta\Delta)\big)^2}\big(\mathrm{e}^{\eta\Delta/2}  - 1\big)^2 \Bigg). \label{eq:overview-fast-subopt-raw}
\end{align}
Given that $\eta\Delta = O(1)$, $\exp(\eta\Delta) = O(1)$ and the $\exp(\eta\Delta)$ in the denominator in~\eqref{eq:overview-fast-subopt-raw} can be ignored and then~\eqref{eq:overview-fast-subopt-raw} becomes
\begin{align}
    & \subopt_1(\pi) + \subopt_2(\pi) \gtrsim \frac{1}{\eta} \log \bigg(1 + \frac{m}{K}\big(\exp(\eta \Delta/2) - 1\big)^2 \bigg)  \gtrsim \frac{1}{\eta} \log \bigg(1 + \frac{m}{K}\eta^2\Delta^2 \bigg)  \gtrsim \frac{m \eta \Delta^2}{K}, \label{eq:overview-fast-subopt-final}
\end{align}
where the second inequality holds due to $e^x-1 \approx x$ and the last holds due to $\log(1 + x) \approx x$. 
Consequently,~\eqref{eq:overview-fast-subopt-final} demonstrates that if the algorithm fails to distinguish between instances that differ on $m$ arms, it suffers a per-step cost of $\Omega(m\eta\Delta^2/K)$.
% This indicates that the cost of error in deciding reward of $m$ arms is $\Omega(m\eta\Delta^2/K)$.
% To sum up, we obtain that for any instances $\bmu, \blambda \in \cV$ and $\pi \in \mathcal{P}(\cA)$.
% \begin{align}
%     \subopt_{\bmu}(\pi) + \subopt_{\blambda}(\pi) \gtrsim  \frac{\eta\Delta^2}{K} d_H(\bmu, \blambda). \label{}
% \end{align}

\paragraph{Minimax Lower Bound of the Suboptimality Gap.}

We show that for $t \geq \eta^2 K$, there exists a choice of $\Delta_t$ such that the suboptimality gap at time step $t$ is $\Omega(\eta K/t)$. Fixing $t\geq \eta^2 K$, we pick $\Delta_t = \sqrt{K/t}$, and use $\bmu \sim_j \blambda$ to denote $d_H(\bmu, \blambda) =1$ and $\mu_j \neq \lambda_j$. As in~\eqref{eq:overview-slow-kl}, we consider the average KL-divergence (up to round $t$) between pairs of instances which differ only on arm $j$:
\begin{align*}
    \frac{1}{|\cV|}\sum_{\bmu \sim_j \blambda} \KL(\PP_{\bmu,t}\|\PP_{\blambda,t}).
\end{align*}
Averaging over $j \in [A]$, one can show that 
\begin{align*}
    \frac{1}{A}\sum_{j=1}^A\frac{1}{|\cV|}\sum_{\bmu \sim_j \blambda} \KL(\PP_{\bmu,t}\|\PP_{\blambda,t})  \leq \frac{2t\Delta_t^2}{A} = 4.
\end{align*}
Consequently, there exist $m=\Omega(K)$ arms for which the corresponding average KL-divergences are $O(1)$, implying that the algorithm \textsf{Alg} cannot reliably distinguish the rewards on these arms. Plugging $m = \Omega(K)$ into~\eqref{eq:overview-fast-subopt-final} gives $\Omega(\eta K/t)$ suboptimality gap.

\definecolor{crimson}{RGB}{220,20,60}

\begin{figure*}[ht]
\centering
\begin{subfigure}{0.32\textwidth}
    \centering
    \includegraphics[width=\linewidth]{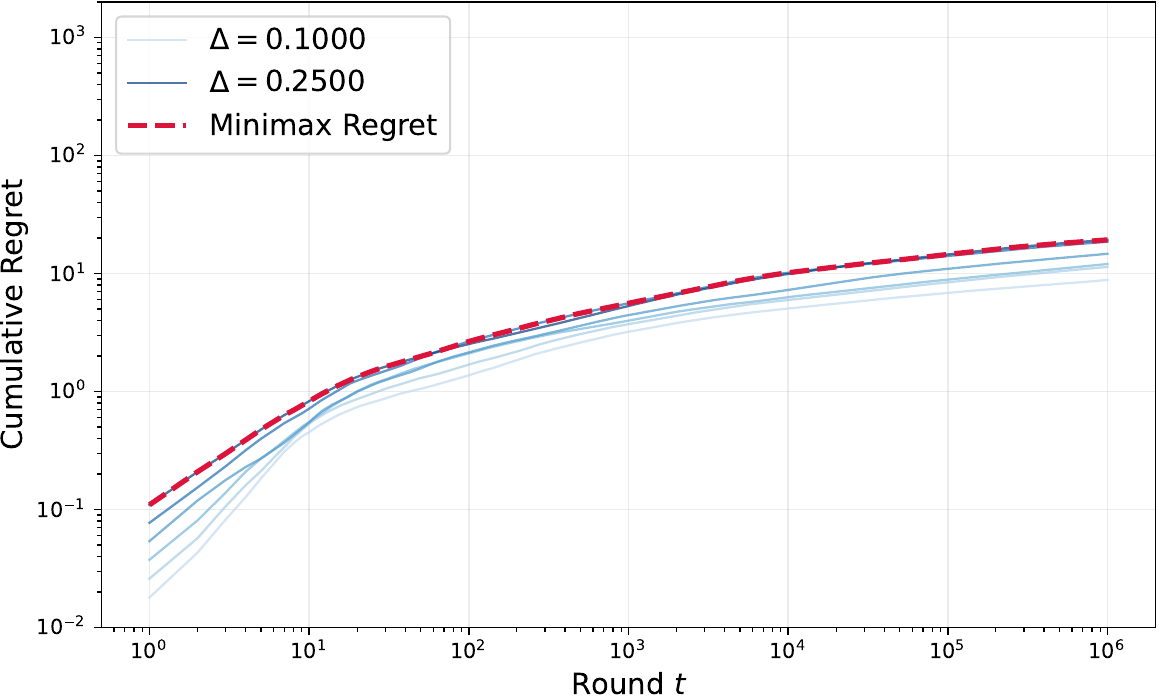}
    \caption{$\eta=5.0$}
    \label{fig:transition-5}
\end{subfigure}
% \hfill
\begin{subfigure}{0.32\textwidth}
    \centering
    \includegraphics[width=\linewidth]{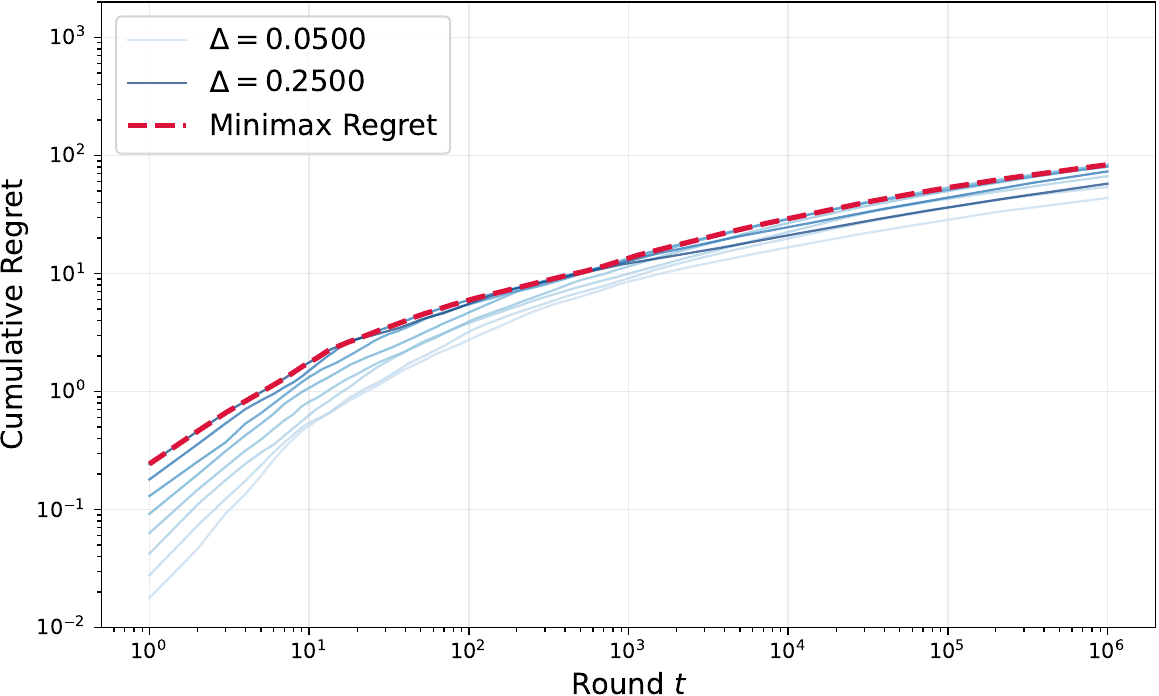}
    \caption{$\eta=20.0$}
    \label{fig:transition-20}
\end{subfigure}
\begin{subfigure}{0.32\textwidth}
    \centering
    \includegraphics[width=\linewidth]{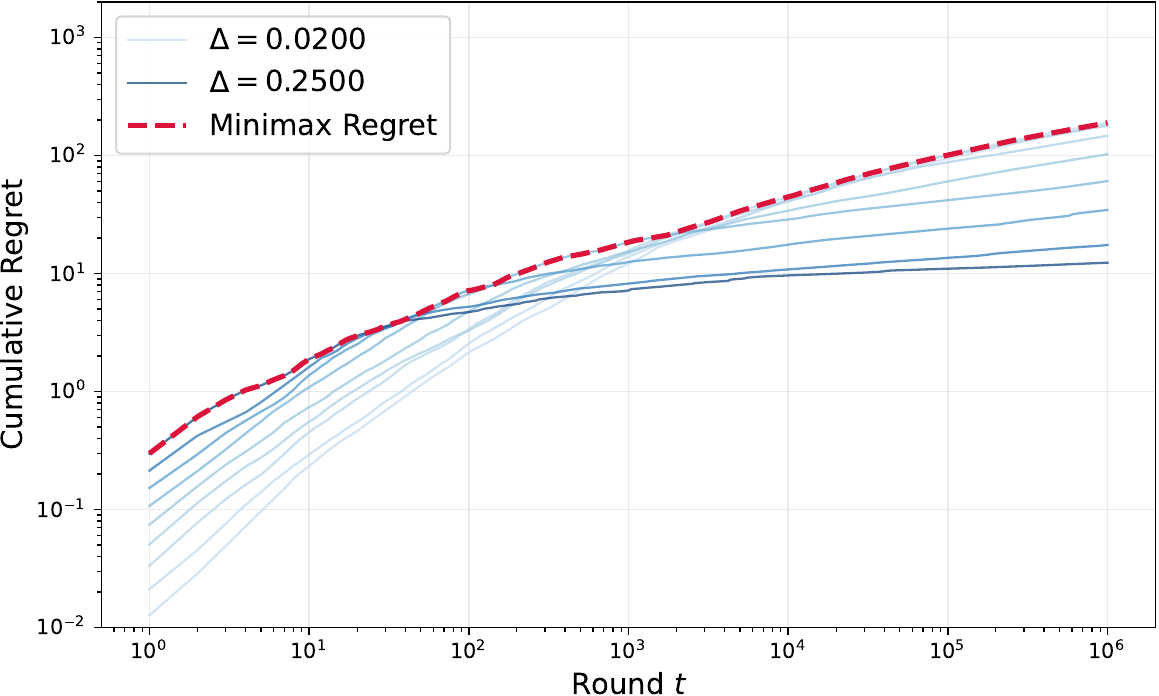}
    \caption{$\eta=100$}
    \label{fig:transition-100}
\end{subfigure}
% \hfill
\begin{subfigure}{0.325\textwidth}
    \centering
    \includegraphics[width=\linewidth]{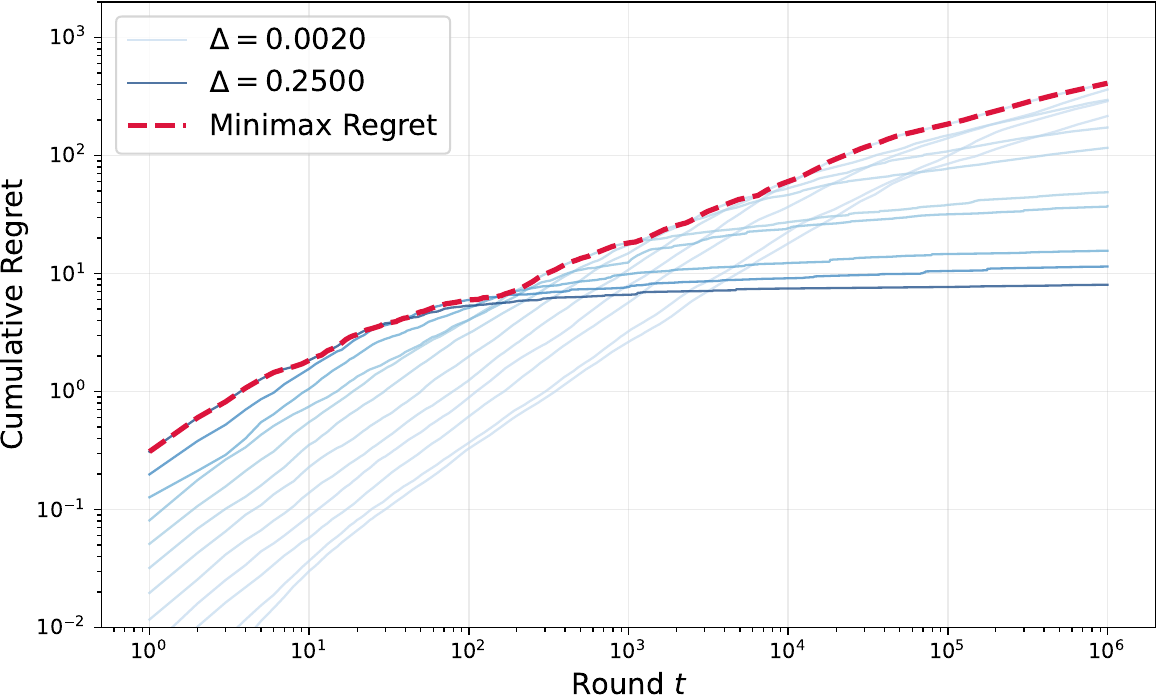}
    \caption{$\eta=500.0$}
    \label{fig:transition-500}
\end{subfigure}
% \vspace{0.5em}
\begin{subfigure}{0.305\textwidth}
    \centering
    \includegraphics[width=\linewidth]{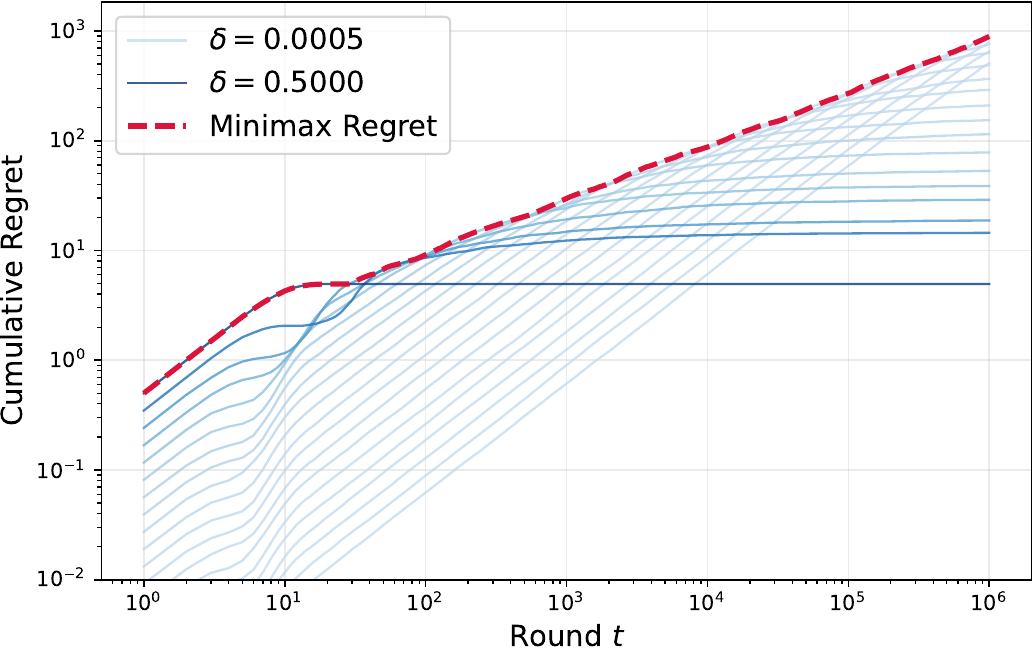}
    \caption{$\eta=\infty$}
    \label{fig:transition-infty}
\end{subfigure}
% \hfill
\begin{subfigure}{0.325\textwidth}
    \centering
    \includegraphics[width=\linewidth]{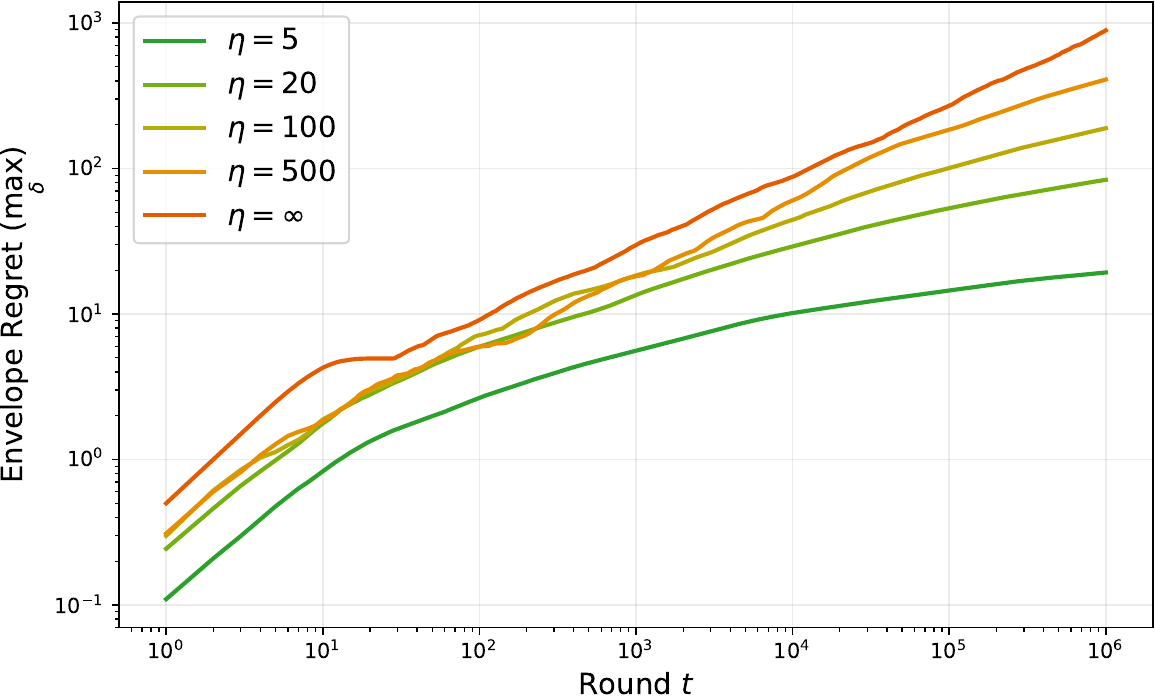}
    \caption{Comparison over $\eta$}
    \label{fig:transition-all}
\end{subfigure}
\caption{The regret curves of KL-UCB for different values of $\eta$, with both axes displayed on log scales. For each fixed $\eta$, the regret corresponding to each instance (determined by $\Delta$) is plotted as a single \textcolor{blue}{blue solid line} in Figures~\ref{fig:transition-5}--\ref{fig:transition-infty}. The \textbf{\textcolor{crimson}{red bold dashed lines}} represent the maximum regret over all instances, and thus approximate the minimax regrets. For ease of comparison, all minimax regret curves are plotted together in Figure~\ref{fig:transition-all}. }
\label{fig:transition}
\end{figure*}

\paragraph{Summing Over Time Steps.}

If we ignore small time steps, directly summing up the $\Omega(\eta K/t)$ suboptimality gap over $t \in [\ceil{\eta^2K}, T]$ yields the expected regret lower bound $\Omega\big(\eta K \log (T/(\eta^2K))\big)$. Such an approach is, however, flawed since the $\Delta_t$ is different for each $t$. This temporal-level discrepancy of the set of hard instances prevents a direct aggregation of these bounds, because the trajectory distribution would be ill-defined if the instances keep evolving as $t$ grows from $1$ to $T$.

To overcome this issue, we construct a single collection of instances that remains invariant for all $t = \Omega(\eta^2 K)$. The idea here is to extend the discrete instance distribution to a continuous distribution. Similar proof ideas have also been applied in previous works~\citep{vovk2001competitive,singer2002universal,zhao2023optimal} to derive $\log T$-type lower bounds. For clarity, we illustrate the idea under $K=2$, in which $\cV = \{\pm 1\}$.

Fixing some $t$ and the corresponding reward gap $\Delta$, the rewards in the previous construction are distributed over $1/2 \pm \Delta$. To make this distribution continuous, we replace $1/2$ with a variable $x$ ranging from $1/2 - \Delta$ to $1/2 + \Delta$. Then $x - \Delta$ exactly scans over $[1/2 - 2\Delta, 1/2]$ and $x + \Delta$ exactly scans over $[1/2, 1/2 + 2\Delta]$, which, collectively, constitutes a uniform coverage of a $4\Delta$-length interval. Moreover, pairing the rewards as $x \pm \Delta$ and applying \eqref{eq:overview-fast-subopt-final} preserves the lower bound:
\begin{align*}
    & \EE_{x \sim \unif([1/2-\Delta,1/2+\Delta])}\EE_{v \in \cV}\EE_{r_{x+v\Delta},t}[\subopt_{r_{x+v\Delta},t}(\pi)] \\
    & \quad = \EE_{u \sim \unif([1/2-2\Delta,1/2+2\Delta])}\EE_{r_u,t}[\subopt_{r_u,t}(\pi)] \\
    & \quad \gtrsim \eta \Delta^2. 
\end{align*}
We then concatenate several consecutive and disjoint copies of the $4\Delta$-length interval to form an interval of length $\alpha$. A uniform distribution of instances over the $\alpha$-length interval still yields an $\Omega(\eta\Delta^2)$ suboptimality gap lower bound.

Now, for each $t \geq \eta^2 K$, we first pick $\Delta_t = \sqrt{K/t}$, and then, if $\alpha$ is sufficiently large, a slight adjustment to $\Delta_t$ enables $\nicefrac{\alpha}{(2\Delta_t)} \in \NN^*$ so that the construction above produces a $t$-independent uniform distribution over an interval of length $\alpha$. This addresses the issue of varying instance distributions across time $t$. Now we can sum over all $t = \Omega(\eta^2 K)$ and obtain the $\Omega(\eta K \log T)$ regret lower bound as desired.

% \color{red}
\section{Numerical Simulations}\label{sec:simulation}

In this section, we conduct numerical simulations to verify the minimax regret guarantees and phase transition presented in Section~\ref{sec:algorithm} and Section~\ref{sec:lower-bound}. We also investigate the policy convergence of Algorithm~\ref{algorithm:kl-ucb} and defer the details to Appendix~\ref{app:additional-simulation}.

In this simulation, we run Algorithm~\ref{algorithm:kl-ucb} on a set of KL-regularized $K$-armed bandits and fix $K = 4$, $T = 10^6$, $\piref=\unif(\cA)$ and rewards $r = [\Delta, 2\Delta, 0, 0]$. To demonstrate how the strength of regularization affects the regret behavior, we consider $\eta \in \{5, 20, 100, 500, \infty\}$, where $\eta = \infty$ corresponds to the standard MAB setting. For each fixed $\eta$, we sweep $\Delta$ log-uniformly over an appropriate interval, and plot the regret curve of each instance (i.e., a choice of $\Delta$). We further take the maximum regret over all instances to approximate the minimax regret. The regret curves are compiled in Figure~\ref{fig:transition}. We observe that the minimax regret curve for the standard MAB setting exhibits a slope of $1/2$ for $t > 100$, indicating a $\sqrt{T}$ rate. In contrast, when $\eta < \infty$, the slope of the regret curve decreases gradually as $t$ increases, with more pronounced decreases for smaller values of $\eta$. These results are consistent with the theoretical guarantees established in the paper.

% The regret curves of KL-UCB for different values of $\eta$. In these experiments, we consider $\eta \in \{5, 20, 100, 500, \infty\}$, where $\eta = \infty$ corresponds to the standard MAB setting. We set $K = 4$ and $T = 10^6$. The bandit instances are defined by $\pi^{\text{ref}} = \mathsf{Unif}(\mathcal{A})$ and rewards $r = [\Delta, 2\Delta, 0, 0]$, where $\Delta$ is swept log-uniformly over an appropriate interval. 
% For each fixed $\eta$, the regret corresponding to each instance (determined by $\Delta$) is plotted as a single \textcolor{blue}{blue solid line} in Figures~\ref{fig:transition-5}--\ref{fig:transition-infty}. The \textbf{\textcolor{crimson}{red bold dashed lines}} represent the maximum regret over all instances, and thus approximate the minimax regrets. For ease of comparison, all minimax regret curves are plotted together in Figure~\ref{fig:transition-all}. \textbf{Both axes are displayed on a log scale.} We observe that the minimax regret curve for the standard MAB setting exhibits a slope of $1/2$ for $t > 100$, indicating a $\sqrt{T}$ rate. In contrast, when $\eta < \infty$, the slope of the regret curve decreases continuously as $t$ increases, with more pronounced decreases for smaller values of $\eta$. These results are consistent with the theoretical guarantees established in the paper.

\section{Conclusion and Future Work}

In this work, we study the MAB problem with a KL-regularized objective and provide a near-complete characterization of the behavior of KL-regularized regret. In particular, we propose a variant of KL-UCB~\citep{zhao2025logarithmic} that achieves a $\tilde{O}(\eta K\log^2 T)$ regret upper bound. This regret is near-optimal, as indicated by an $\Omega(\eta K\log T)$ regret lower bound. Furthermore, in the low regularization regime, our theoretical analysis shows an $\tilde{\Theta}(\sqrt{KT\log T})$ regret with matching bounds, providing a comprehensive understanding of the KL-regularized objectives for \emph{online} learning in MABs. 

Currently, there is still a $\Theta(\log T)$ gap between our upper and lower bounds. Moreover, our analysis is restricted to the tabular setting with finitely many arms and stochastic rewards. Fully closing the gap and extending these near-matching results to structured settings such as contextual bandits~\citep{chu2011contextual}, bandits with linear or general function approximation~\citep{abbasi2011improved,russo2013eluder} and decision making in the face of adversaries~\citep{auer2002nonstochastic} are interesting directions for future work.

% \section*{Acknowledgements}

% We thank the anonymous reviewers and area chair for their helpful comments. We also thank Weitong~Zhang and Xuheng~Li for helpful discussions at the early stage of this project.
% KJ, QZ, HZ, QD and QG are supported in part by the National Science Foundation DMS-2323113, IIS-2403400 and CPS-2312094. 
% The views and conclusions contained in this paper are those of the authors and should not be interpreted as representing any funding agencies.

\appendix

\section{Additional Simulations}\label{app:additional-simulation}

In this section, we present some empirical policy convergence results of Algorithm~\ref{algorithm:kl-ucb}. In these simulations, we consider the $K$-armed bandits with $K = 5$, $T = 10^6$, $\piref=\unif(\cA)$ and rewards randomly sampled from $\unif([0,1]^5)$. To characterize the effect of the regularization strength, we vary $\eta \in \{2, 10, 50\}$. For each $\eta$, we plot the TV distance between the learned policy $\pi_t$ and three comparison policies: the reference policy $\piref$, the optimal policy $\pi^*_{\mathrm{KL}}$ for the KL-regularized regret, and the deterministic optimal policy $\pi^*_{\mathrm{std}}$ for the standard regret. As shown in Figure~\ref{fig:distance}, for small $\eta$ (strong regularization), $\pi_t$ remains close to $\piref$ and converges toward $\pi^*_{\mathrm{KL}}$, consistent with the theoretical prediction that the KL penalty dominates the objective. As $\eta$ increases, the regularization weakens and the policy drifts away from $\piref$, instead approaching $\pi^*_{\mathrm{std}}$, thereby recovering the unregularized optimum. This behavior illustrates how $\eta$ interpolates between the two regimes, complementing our theoretical guarantees on policy convergence.

\begin{figure*}[ht]
\centering   
\begin{subfigure}{0.32\textwidth}
    \centering
    \includegraphics[width=\linewidth]{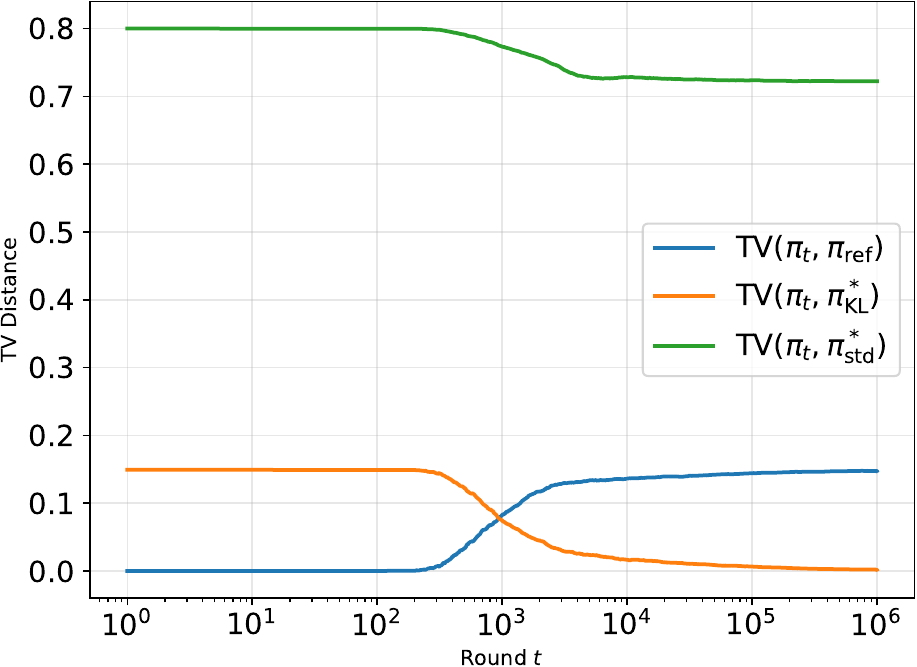}
    \caption{$\eta=2.0$}
    \label{fig:dist-2}
\end{subfigure}
\begin{subfigure}{0.32\textwidth}
    \centering
    \includegraphics[width=\linewidth]{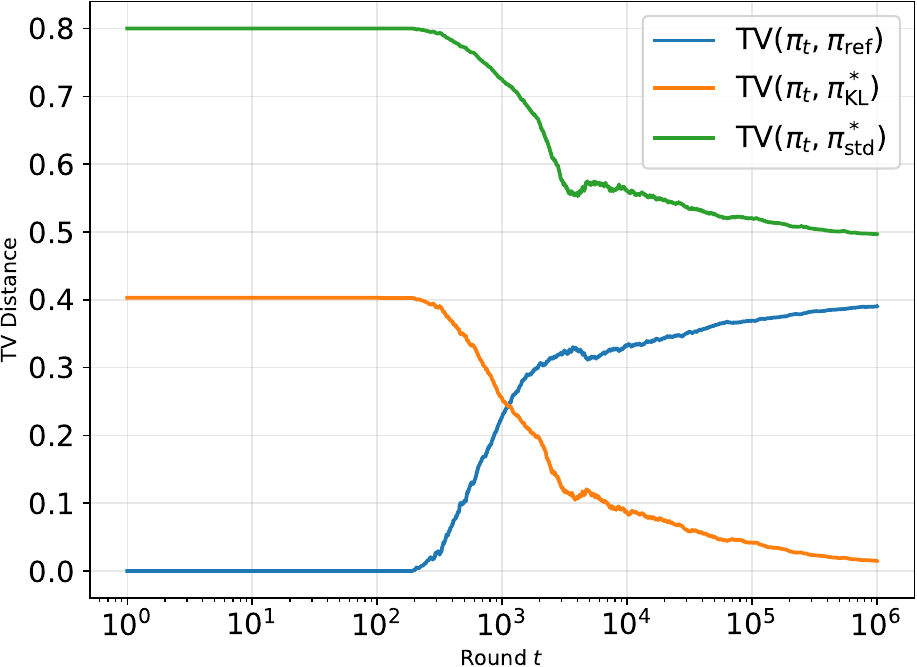}
    \caption{$\eta=10.0$}
    \label{fig:dist-10}
\end{subfigure}
% \begin{subfigure}{0.24\textwidth}
%     \centering
%     \includegraphics[width=\linewidth]{pic/tv_distance_20.0.pdf}
%     \caption{$\eta=20.0$}
%     \label{fig:dist-20}
% \end{subfigure}
\begin{subfigure}{0.32\textwidth}
    \centering
    \includegraphics[width=\linewidth]{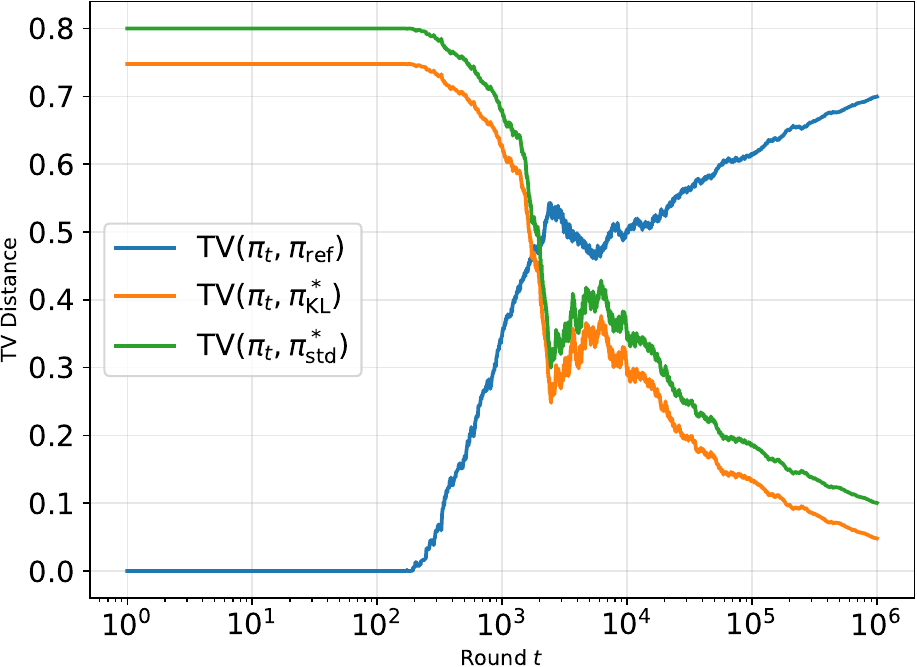}
    \caption{$\eta=50.0$}
    \label{fig:dist-50}
\end{subfigure}
\caption{Policy convergence in TV distance, with $\eta \in \{2, 10, 50\}$. As $\eta$ increases, the regularization weakens, making the policy drift away from the reference policy $\piref$ and approach the optimal policy under the standard unregularized objective $\pi^*_{\mathrm{std}}$. }
\label{fig:distance}
\end{figure*}

\section{Missing Proof in Section~\ref{sec:algorithm}}\label{sec:missing-proof-of-upper-bounds}

\subsection{Proof of Lemma~\ref{lem:optimistic}}

\begin{proof}[Proof of Lemma~\ref{lem:optimistic}]
The proof is standard and we present it here for completeness. Fix a time step $t$ and a specific arm $a$, by Hoeffding's inequality (Lemma~\ref{lem:azuma-hoeffding}), with probability at least $1 - 2\delta/(KT)$, we know that
\begin{align*}
    \bigg| \fgt(a) -  \frac{1}{N_t(a) \vee 1} \sum_{i=1}^t r_i \ind \{a_i = a\} \bigg| \leq \sqrt{\frac{2\log(KT/\delta)}{N_t(a) \vee 1}} = b_t(a).
\end{align*}
Taking a union bound over all $t \in [T]$ and $a \in \cA$ finishes the proof.\footnote{The fact that $N_t(a)$ is itself a random variable seemingly prevents the application of Hoeffding's inequality, which is also a standard caveat; we refer the readers to, e.g., \citet[Section~11.2.3]{orabona2019modern} for details.}
\end{proof}

\subsection{Proof of Theorem~\ref{thm:upperbound}}
\begin{proof}[Proof of Theorem~\ref{thm:upperbound}]
The proof follows the previous proof in~\citet{zhao2025logarithmic}. We first prove the ``fast rate'' when $\eta$ is small. The following lemma gives the KL-regularized regret decomposition.

\begin{lemma}[Lemma A.1, \citealt{zhao2025logarithmic}]\label{lem:taylor-expansion-optimism}
Let $\fop$ be an optimistic estimator of the ground truth reward $\fgt$, i.e., $\fop(a) \geq \fgt(a)$ for all $a \in \cA$. Let $\pihat(a) \propto  \piref(a) \exp \big(\eta \cdot \fop(a)\big)$ and $\pi^*(a) \propto  \piref(a) \exp \big(\eta \cdot \fgt(a)\big)$, then
\begin{align*}
    J(\pi^*) - J(\pihat) \leq \eta \EE_{a \sim \pihat}\big[\big(\fop(a) - \fgt(a)\big)^2\big].
\end{align*}
\end{lemma}

Now we are ready to prove the ``fast rate'' upper bound. 
On the high-probability event $\cE(\delta)$, we can decompose the regret by
\begin{align}
\notag
    \regret(T) &= \sum_{t=1}^T \big[ J(\pi^*) - J(\pi_t) \big]\\\notag
    &\leq \sum_{t=1}^T \eta \EE_{a \sim \pi_t}\big[\big(\fop_{t-1}(a) - \fgt(a)\big)^2\big]\\\label{eq:0000}
    &\leq \eta \sum_{t=1}^T \EE_{a \sim \pi_t} \bigg[ \frac{8\log(KT/\delta)}{N_{t-1}(a) \vee 1} \bigg],
\end{align}
where the first inequality holds due to Lemma~\ref{lem:taylor-expansion-optimism} and the second inequality is by the definition of event $\cE(\delta)$. 
To obtain a high-probability upper bound for $\regret(T)$, we conduct the following decomposition
\begin{align}
\label{eq:0001}
    \sum_{t=1}^T \EE_{a \sim \pi_t} \bigg[ \frac{1}{N_{t-1}(a)\vee 1} \bigg] &= \underbrace{\sum_{t=1}^T \bigg[ \frac{1}{N_{t-1}(a_t) \vee 1} \bigg]}_{I_1} + \underbrace{\sum_{t=1}^T\bigg( \EE_{a \sim \pi_t} \bigg[ \frac{1}{N_{t-1}(a) \vee 1} \bigg]-  \frac{1}{N_{t-1}(a_t) \vee 1} \bigg)}_{I_2}
\end{align}
For $I_1$, we can bound it as follows:
\begin{align}
\notag
I_1 &= \sum_{t=1}^T \bigg[ \frac{1}{N_{t-1}(a_t) \vee 1} \bigg]\\\notag
&\le \sum_{a \in \cA} \bigg(1 +\sum_{i=1}^{N_{T-1}(a)}  \frac{1}{i} \bigg)\\\label{eq:0003}
&\le \sum_{a \in \cA} \bigg(1 +\sum_{i=1}^{N_{T-1}(a)}  2\log \Big(1 + \frac{1}{i}\Big)\bigg),
\end{align}
where the last inequality holds due to $x \le 2\log(1+x)$ when $0 < x \le 1$. To move on, we have
\begin{align}
\notag
    \sum_{a \in \cA} \bigg(1 +\sum_{i=1}^{N_{T-1}(a)}  2\log \Big(1 + \frac{1}{i}\Big)\bigg) &= \sum_{a \in \cA} \bigg(1 +\sum_{i=1}^{N_{T-1}(a)}  2\log \Big(\frac{i+1}{i}\Big)\bigg)\\\notag
    &= \sum_{a \in \cA} \bigg(1 +  2 \log \bigg[\prod_{i=1}^{N_{T-1}(a)}\frac{i+1}{i}\bigg]\bigg)\\\notag
    &= \sum_{a \in \cA} \bigg(1 +  2 \log \big(N_{T-1}(a)+1\big)\bigg)\\\label{eq:0002}
    &\le 4K \log T, 
\end{align}
where the last inequality holds due to $N_{T-1}(a) \le T, \forall a\in\cA$. Thus, we know $I_1 \le 4K \log T$. For $I_2$, let $x_t = \big( \EE_{a \sim \pi_t} \big[ {1}/({N_{t-1}(a) \vee 1}) \big]-  {1}/({N_{t-1}(a_t) \vee 1}) \big)$. Let $\cF_t = \sigma(a_1,r_1,a_2,r_2,\ldots,a_{t},r_{t})$ be the $\sigma$-algebra generated by the actions and rewards up to time $t$. Then, we know $x_t$ is $\cF_t$-measurable and $\EE[x_t | \cF_{t-1}] = 0$.
% we need to first derive a tight upper bound for the gap between $\sum_{t = 1}^T \frac{8\eta \log(KT/\delta)}{N_{t-1}(a_t)}$ and $\sum_{t = 1}^T \EE_{a \sim \pi_t} \bigg[\frac{8\eta \log(KT/\delta)}{N_{t-1}(a)}\bigg]$. 

Let $\cE_i(\tau) = \big\{\sum_{t = 1}^\tau \EE_{a \sim \pi_t} \big[{1}/({N_{t-1}(a) \vee 1})\big] \le 2^i \big\}$. Then, $\ind\big(\cE_i(t)\big)$ is $\cF_{t-1}$-measurable. Thus, $\EE[x_t \ind\big(\cE_i(t)\big) |\cF_{t-1}] = \ind\big(\cE_i(t)\big)\EE[x_t | \cF_{t-1}] = 0$. Moreover, we have
\begin{align*}
    \EE\Big[\big(x_t\ind[\cE_i(t)]\big)^2 | \cF_{t-1}\Big] &= \EE\Big[x_t^2\ind[\cE_i(t)] | \cF_{t-1}\Big]\\
    &=\ind[\cE_i(t)] \cdot \EE\bigg[\Big( \EE_{a \sim \pi_t} \Big[ \frac{1}{N_{t-1}(a) \vee 1} \Big]-  \frac{1}{N_{t-1}(a_t) \vee 1} \Big)^2 \bigg| \cF_{t-1}\bigg]\\
    &= \ind[\cE_i(t)] \cdot \bigg(\EE_{a\sim \pi_t}\bigg[\Big( \frac{1}{N_{t-1}(a) \vee 1} \Big)^2 \bigg] - \Big( \EE_{a \sim \pi_t} \Big[ \frac{1}{N_{t-1}(a) \vee 1} \Big]\Big)^2\bigg)\\
    &\le \ind[\cE_i(t)] \cdot \EE_{a\sim \pi_t}\bigg[\Big( \frac{1}{N_{t-1}(a) \vee 1} \Big)^2 \bigg]\\
    &\le \ind[\cE_i(t)] \cdot \EE_{a\sim \pi_t}\bigg[ \frac{1}{N_{t-1}(a) \vee 1} \bigg],
\end{align*}
where the first inequality holds as we drop the nonpositive term. The second inequality holds due to $1/(N_{t-1}(a_t) \vee 1) \le 1$. Therefore, we have
\begin{align*}
    \sum_{s=1}^t \EE\Big[\big(x_s\ind[\cE_i(s)]\big)^2 | \cF_{s-1}\Big] &\le \sum_{s=1}^t \ind[\cE_i(s)] \cdot \EE_{a\sim \pi_s}\bigg[ \frac{1}{N_{s-1}(a) \vee 1} \bigg].
\end{align*}
Let $\tau_i := \max\big\{t \in [T]: \sum_{s=1}^t \EE_{a \sim \pi_s} \big[{1}/({N_{s-1}(a) \vee 1})\big] \le 2^i\big\}$. If $t \le \tau_i$, $\ind(\cE_i(s)) = 1$ for any $ s \le t$; which means
\begin{align}
 t \leq \tau_i \Longrightarrow   \sum_{s=1}^t \ind[\cE_i(s)] \cdot \EE_{a\sim \pi_s}\bigg[ \frac{1}{N_{s-1}(a) \vee 1} \bigg] &= \sum_{s=1}^t \EE_{a\sim \pi_s}\bigg[ \frac{1}{N_{s-1}(a) \vee 1} \bigg]\le 2^i. \label{eq:divide-and-conquer-case-less}
\end{align}
The inequality holds due to $\ind(\cE_i(t)) = 1$. Otherwise, if $t > \tau_i$, we have
\begin{align}
 t > \tau_i \Longrightarrow  \sum_{s=1}^t \ind[\cE_i(s)] \cdot \EE_{a\sim \pi_s}\bigg[ \frac{1}{N_{s-1}(a) \vee 1} \bigg] &= \sum_{s=1}^{\tau_i}\ind[\cE_i(s)] \EE_{a\sim \pi_s}\bigg[ \frac{1}{N_{s-1}(a) \vee 1} \bigg] \notag\\
    &\qquad + \sum_{s=\tau_i+1}^{t} \ind[\cE_i(s)] \EE_{a\sim \pi_s}\bigg[ \frac{1}{N_{s-1}(a) \vee 1} \bigg] \notag\\
    &= \sum_{s=1}^{\tau_i}\EE_{a\sim \pi_s}\bigg[ \frac{1}{N_{s-1}(a) \vee 1} \bigg] \notag\\
    &\le 2^i, \label{eq:divide-and-conquer-case-gt}
\end{align}
where we use $\ind(\cE_i(s)) = 1,\forall s \le \tau_i$ and $\ind(\cE_i(s)) = 0,\forall s > \tau_i$; the last inequality holds due to $\ind(\cE_i(\tau_i)) = 1$. Therefore, we always have
\begin{align*}
    \sum_{s=1}^t \EE\Big[\big(x_s\ind[\cE_i(s)]\big)^2 | \cF_{s-1}\Big] &\le 2^i.
\end{align*}
Using Freedman's inequality (Lemma \ref{lem:freedman-m}), we have for any $i$, with probability at least $1 - \delta/(\lceil \log_2 T\rceil)$, the following inequality holds:
\begin{align*}
    &-\sum_{t = 1}^T \frac{1}{N_{t - 1}(a_t) \vee 1} \cdot \ind(\cE_i(t)) + \sum_{t = 1}^T \EE_{a \sim \pi_t} \bigg[\frac{1}{N_{t - 1}(a) \vee 1} \ind(\cE_i(t))\bigg] \\&\le \sqrt{2 \cdot 2^i \log(\lceil \log_2 T\rceil / \delta)} + 2 / 3 \cdot \log(\lceil \log_2 T\rceil / \delta).
\end{align*}

Taking the union bound, we have with probability at least $1-\delta$, the above inequality holds for any $1 \le i \le \lceil \log_2 T \rceil$. We take $i = \max\big\{1, \big\lceil \log_2 \sum_{t = 1}^T \EE_{a \sim \pi_t} [{1}/{(N_{t - 1}(a) \vee 1 )}] \big\rceil\big\} \le \lceil \log_2 T \rceil $, so that $2^i \le 2\sum_{t = 1}^T \EE_{a \sim \pi_t} [{1}/{(N_{t - 1}(a) \vee 1 )}]$ since the sum is at least $1$. Then, $\ind(\cE_i(t)) = 1$ holds for any $t \le T$. This gives us 
\begin{align}
\notag
    I_2 &= -\sum_{t = 1}^T \frac{1}{N_{t - 1}(a_t) \vee 1} + \sum_{t = 1}^T \EE_{a \sim \pi_t} \bigg[\frac{1}{N_{t - 1}(a) \vee 1}\bigg]\\\label{eq:0004}
    &
    \le \sqrt{4 \cdot \sum_{t = 1}^T \EE_{a \sim \pi_t} \bigg[\frac{1}{N_{t - 1}(a) \vee 1} \bigg] \cdot \log(\lceil \log_2 T \rceil / \delta)} + 2 / 3 \cdot \log(\lceil \log_2 T \rceil / \delta).
\end{align}
Substituting \eqref{eq:0002} and \eqref{eq:0004} into \eqref{eq:0001}, we have
\begin{align*}
    \sum_{t=1}^T \EE_{a \sim \pi_t} \bigg[ \frac{1}{N_{t-1}(a)\vee 1} \bigg] &\le \sqrt{4 \cdot \sum_{t = 1}^T \EE_{a \sim \pi_t} \bigg[\frac{1}{N_{t - 1}(a) \vee 1} \bigg] \cdot \log(\lceil \log_2 T \rceil / \delta)} \\
    &\qquad + 4K \log T+ 2 / 3 \cdot \log(\lceil \log_2 T \rceil / \delta).
\end{align*}
Using $x \le a \sqrt{x} + b \Rightarrow x \le a^2 + 2b$, we have
\begin{align}
\label{eq:0005}
    \sum_{t=1}^T \EE_{a \sim \pi_t} \bigg[ \frac{1}{N_{t-1}(a)\vee 1} \bigg] \le 6\log(\lceil \log_2 T \rceil / \delta) + 8K \log T.
\end{align}
Substituting \eqref{eq:0005} into \eqref{eq:0000}, we know that with probability at least $1-3\delta$, the following inequality holds:
\begin{align*}
    \regret(T) &\le 8\eta \log(KT/\delta)\Big[6\log(\lceil \log_2 T \rceil / \delta) + 8K \log T\Big]\\
    &\le O\big(\eta K \cdot \log^2(KT/\delta)\big) .
\end{align*}
% Rearanging the inequality above, it holds that \begin{align*}
%     &\sum_{t = 1}^T \EE_{a \sim \pi_t} \bigg[\frac{1}{N_{t - 1}(a) \vee 1}\bigg] - \sqrt{4 \cdot \sum_{t = 1}^T \EE_{a \sim \pi_t} \bigg[\frac{1}{N_{t - 1}(a) \vee 1} \bigg] \cdot \log(\lceil \log_2 T \rceil / \delta)} \\
%     &\qquad \le \sum_{t = 1}^T \frac{1}{N_{t - 1}(a_t) \vee 1} + 2 / 3 \cdot \log(\lceil \log_2 T \rceil / \delta)
%     \\&\qquad \le 4K \log T + 2 / 3 \cdot \log(\lceil \log_2 T \rceil / \delta),
% \end{align*}
% where the last inequality holds using the same calculation as \eqref{eq:0003} and \eqref{eq:0002}.
 
In the next step, we consider the slow rate. Still, the following proof is conditioned on $\cE(\delta)$. The regret can be decomposed as follows:
\begin{align}
\notag
    \regret(T) &= \sum_{t=1}^T \Bigg[ \EE_{a \sim \pi^*} \bigg[r(a) - \eta^{-1} \log \frac{\pi^*(a)}{\piref(a)} \bigg] - \EE_{a \sim \pi_t} \bigg[r(a) - \eta^{-1} \log \frac{\pi_t(a)}{\piref(a)} \bigg] \Bigg] \\\notag
    & \leq \sum_{t=1}^T \Bigg[ \EE_{a \sim \pi^*} \bigg[\fop_{t-1}(a) - \eta^{-1} \log \frac{\pi^*(a)}{\piref(a)} \bigg] - \EE_{a \sim \pi_t} \bigg[r(a) - \eta^{-1} \log \frac{\pi_t(a)}{\piref(a)} \bigg] \Bigg] \\\notag
    & \leq \sum_{t=1}^T \Bigg[ \EE_{a \sim \pi_t} \bigg[\fop_{t-1}(a) - \eta^{-1} \log \frac{\pi_t(a)}{\piref(a)} \bigg] - \EE_{a \sim \pi_t} \bigg[r(a) - \eta^{-1} \log \frac{\pi_t(a)}{\piref(a)} \bigg] \Bigg] \\\notag
    & = \sum_{t=1}^T \EE_{a \sim \pi_t} [\fop_{t-1}(a) - \fgt(a)] \\\label{eq:0006}
    & \leq \sum_{t=1}^T \EE_{a \sim \pi_t} \Bigg[ 2\sqrt{\frac{2\log(TK/\delta)}{N_{t-1}(a) \vee 1}}\Bigg],
\end{align}
where the first inequality holds since $\fop_{t-1}$ is optimistic on event $\cE(\delta)$, the second inequality holds since $\pi_t$ is optimal under $\fop_{t-1}$, and the last inequality holds on event $\cE(\delta)$. We have 
\begin{align}
\label{eq:0007}
    \sum_{t=1}^T \EE_{a \sim \pi_t} \bigg[ \frac{1}{\sqrt{N_{t-1}(a)\vee 1}} \bigg] & = \underbrace{\sum_{t=1}^T \bigg[ \frac{1}{\sqrt{N_{t-1}(a_t) \vee 1}} \bigg]}_{J_1} \notag \\
    & \quad + \underbrace{\sum_{t=1}^T\bigg( \EE_{a \sim \pi_t} \bigg[ \frac{1}{\sqrt{N_{t-1}(a) \vee 1}} \bigg]-  \frac{1}{\sqrt{N_{t-1}(a_t) \vee 1}} \bigg)}_{J_2}.
\end{align}
For $J_1$, we have
\begin{align}
\notag
    \sum_{t=1}^T \bigg[ \frac{1}{\sqrt{N_{t-1}(a_t) \vee 1}} \bigg] &\le \sum_{a \in \cA} \bigg[ 1 + \sum_{i=1}^{N_{T-1}(a)} \frac{1}{\sqrt{i}}\bigg]\\\notag
    &\le \sum_{a\in \cA} \Big[ 1 + 2\sqrt{N_{T-1}(a)}\Big]\\\notag
    &\le K + 2\sqrt{K \sum_{a\in \cA} {N_{T-1}(a)}}\\\label{eq:0008}
    &\le K + 2\sqrt{KT},
\end{align}
where the second inequality holds due to $\sum_{i=1}^N 1/\sqrt{i} \le \sum_{i=1}^N \int_{i-1}^{i} u^{-1/2} \ud u = 2\sqrt{N}$, the third inequality holds due to the Cauchy--Schwarz inequality, and the last inequality holds due to $\sum_{a\in \cA} N_{T-1}(a) = T-1$. For $J_2$, we apply Lemma \ref{lem:azuma-hoeffding}. Then, with probability at least $1-\delta$, we have
\begin{align}
\notag
    J_2 &= \sum_{t=1}^T\bigg( \EE_{a \sim \pi_t} \bigg[ \frac{1}{\sqrt{N_{t-1}(a) \vee 1}} \bigg]-  \frac{1}{\sqrt{N_{t-1}(a_t) \vee 1}} \bigg)\\\label{eq:0009}
    &\le 2\sqrt{2T\log(1/\delta)}.
\end{align}
Substituting \eqref{eq:0008} and \eqref{eq:0009} into \eqref{eq:0007}, we have
\begin{align*}
    \sum_{t=1}^T \EE_{a \sim \pi_t} \bigg[ \frac{1}{\sqrt{N_{t-1}(a)\vee 1}} \bigg] \le K + 2\sqrt{KT} + 2\sqrt{2T\log(1/\delta)}.
\end{align*}
Combining this with \eqref{eq:0006}, we have with probability at least $1-3\delta$,
\begin{align*}
    \regret(T) &\le 2\sqrt{2 \log(TK/\delta)} \big[K + 2\sqrt{KT} + 2\sqrt{2T\log(1/\delta)}\big]\\
    &\le \tilde O(K +\sqrt{KT}),
\end{align*}
which finishes the proof of the regret upper bound under low-regularization.
\end{proof}

\section{Missing Proof in Section~\ref{sec:lower-bound}}

\subsection{Proof of Theorem~\ref{thm:lowerbound-slow}}

\begin{proof}[Proof of Theorem~\ref{thm:lowerbound-slow}]

The construction of instances follows \citet[Chapter~15]{lattimore2020bandit}. We fix $\eta$, $K \geq 9$, $T$, let $\cA = [K]$ and select $\piref = \unif(\cA)$. Given any reward function $r: \cA \to [0,1]$, we also use $r$ to denote the corresponding bandit instance $([K], r, \eta, \piref, T)$ when there is no ambiguity. Now we take $r_1: \cA \to [0,1]$ and $r_1(i) = \Delta \ind\{i=1\}$, where $\Delta > 0$ is some parameter to be figured out later. Given a fixed algorithm \textsf{Alg}, we use $\PP_1$ and $\EE_1$ to denote the trajectory distribution jointly given by $r_1$ and \textsf{Alg}. Recall that $N_t(j)$ is the count of times the $j$-th arm has been pulled up to step $t$. Let 
\begin{align*}
    i_1 = \argmin_{j > 1} \EE_1 [N_T(j)].
\end{align*}
Without loss of generality, we assume $i_1=2$. By the pigeonhole principle, we know that $\EE_1[N_T(2)] \leq T/(K-1)$. Now we consider the second instance given by $r_2: \cA \to [0,1]$, such that $r_2(2) = 2 \Delta$ and $r_2(j) = r_1(j)$ for all $j \neq 2$. We now compute $\pi_1^*$ and $\pi_2^*$, which are the optimal policies under $r_1$ and $r_2$. Direct computation gives
\begin{align*}
    \pi_1^* (1) = \frac{\exp (\eta\Delta)}{\exp(\eta\Delta) + K-1}, \text{ and } \pi_1^*(i) = \frac{1}{\exp(\eta\Delta) + K-1} \text{ for all } i > 1,
\end{align*}
and
\begin{align*}
    \pi_2^* (1) = \frac{\exp (\eta\Delta)}{\exp(\eta\Delta) + \exp(2\eta\Delta) + K-2}, \quad  \pi_2^* (2) = \frac{\exp (2\eta\Delta)}{\exp(\eta\Delta) + \exp(2\eta\Delta) + K-2},
\end{align*}
and 
\begin{align*}
    \pi_2^* (i) = \frac{1}{\exp(\eta\Delta) + \exp(2\eta\Delta) + K-2}, \ \forall \ i > 2.
\end{align*}
For any policy $\pi \in \mathcal{P}(\cA)$, we consider the suboptimality gap $\subopt_{r_1}(\pi, \pi^*_1) + \subopt_{r_2}(\pi, \pi^*_2)$. For simplicity, we use $\subopt_1(\pi)$ to denote $\subopt_{r_1}(\pi, \pi^*_1)$ and $\subopt_2(\pi)$ for $\subopt_{r_2}(\pi, \pi^*_2)$, correspondingly. By Lemma~\ref{lem:kl-subopt-decompose},
\begin{align}
    \subopt_1(\pi) + \subopt_2(\pi) = \eta^{-1} \big[\KL(\pi \| \pi_1^*) + \KL(\pi\| \pi_2^*) \big]. \label{eq:slow-lb:separation-expression}
\end{align}
It is known that the unique minimizer of \cref{eq:slow-lb:separation-expression} is $\pihat(a) \propto \sqrt{\pi_1^*(a)\pi^*_2(a)}$~\citep[(B.9)]{zhao2026towards}, which gives
\begin{align*}
    & \pihat(1) = \pihat(2) \propto {\exp(\eta\Delta)}, \quad\text{and}\quad \pihat(i) \propto {1}, \quad \forall \ i>2.
\end{align*}
Therefore, we know that
\begin{align*}
    \eta \big(\subopt_1(\pihat) + \subopt_2(\pihat)\big) =  \log \frac{(\exp(\eta\Delta) + K-1)(\exp(\eta\Delta) + \exp(2\eta\Delta) + K-2)}{(2\exp(\eta\Delta) + K-2)^2}.
\end{align*}
Now we select $\Delta = 2\sqrt{K/T}$. Then by the fact that $T \leq \eta^2 K/\log^2 K$, we have $\eta \Delta \geq 2\log K$, and consequently $e^{\eta \Delta} \geq  K$, which gives that
\begin{align*}
    \subopt_1(\pihat) + \subopt_2(\pihat) & \geq \eta^{-1} \log \frac{(\exp(\eta\Delta) + K-1)(\exp(\eta\Delta) + \exp(2\eta\Delta) + K-2)}{(2\exp(\eta\Delta) + K-2)^2} \\
    & \geq  \eta^{-1} \log \frac{\exp(2\eta\Delta)(1 + \exp(\eta\Delta))}{9\exp(2\eta\Delta)} \\
    & \geq \eta^{-1} (\eta \Delta - \log 9) \\
    & \geq \Delta /2,
\end{align*}
where the second inequality holds due to $9 \leq K \leq e^{\eta\Delta}$ and the last inequality holds due to $\log 9 \leq \log K \leq \eta\Delta/2$. Now, applying Lemma~\ref{lem:two-point}, we obtain that
\begin{align}
    \inf_{\textsf{Alg}} \sup_{r \in \cR} \EE_{\cD \sim \PP_{r, \textsf{Alg}}}\big[ \regret(T) \big] \geq \frac{T\Delta}{8}\cdot \exp\Big(-\kl{\PP_1}{\PP_2} \Big), \label{eq:apply-le-cam}
\end{align}
where we recall that $\PP_{r, \textsf{Alg}}$ is the trajectory distribution of $\textsf{Alg}$ interacting with instance $r$, and $\PP_{\ell} \coloneqq \PP_{r_\ell, \textsf{Alg}}$.
By the divergence decomposition lemma~\citep[Lemma~15.1]{lattimore2020bandit},
\begin{align*}
    \kl{\PP_1}{\PP_2} = \sum_{k=1}^K \EE_1[N_T(k)]\KL(r_1(k), r_2(k)) = \EE_1[N_T(2)]\KL(0, 2\Delta) \leq \frac{2T\Delta^2}{K-1},
\end{align*}
where the inequality holds due to $\EE_1[N_T(2)] \leq T/(K-1)$ and \Cref{fact:kl-gaussian}. Recalling that $\Delta = 2\sqrt{K/T}$, we know that $\kl{\PP_1}{\PP_2} \leq 8K/(K-1) \leq 9$. Substituting them into \cref{eq:apply-le-cam}, we obtain 
\begin{align*}
    \inf_{\textsf{Alg}} \sup_{r \in \cR} \EE_{r} \regret(T) = \Omega(\sqrt{KT}),
\end{align*}
where $\EE_{r}$ denotes the expectation with respect to the trajectory distribution induced by $\textsf{Alg}$ interacting with instance $r$.
\end{proof}

\subsection{Proof of \Cref{thm:lowerbound-fast}}

\begin{proof}[Proof of Theorem~\ref{thm:lowerbound-fast}]
The construction below uses $2K$ arms. Applying it with $\max\{1, \lfloor K/2 \rfloor\}$ in place of $K$ yields the claimed bound for any $K \geq 2$, up to a constant factor.
We consider the following instance class. Given $K$, $\eta$ and $\piref = \unif(\cA)$ and fix some algorithm \textsf{Alg}, we consider $\cA = [2K]$ and consider a class of reward functions parameterized by some $(\xb,\bmu)$, where $\xb \in \RR^K$ and $\bmu \in \cV = \{\pm1\}^K$, such that the mean reward $r_{\xb,\bmu}(i) = 1/2 + x_i + \mu_i\Delta$ for all $i \leq K$ and $r_{\xb,\bmu}(i) = 1/2 + \alpha$ for all $K < i \leq 2K$. Here $\alpha \geq 2\Delta > 0$ are parameters to be assigned later \emph{subject to $\nicefrac{\alpha}{2\Delta} \in \NN^*$}. Let the reward noises follow $\iid$ standard Gaussian, which satisfy our 1-sub-Gaussian assumption on $\{\varepsilon_t\}_{t\geq1}$.
Given $\xb \in \RR^K$ and $\bmu \in \cV$, we use $(\xb, \bmu)$ to denote the bandit instance $(\cA, r_{\xb, \bmu}, \eta, \piref, T)$.

\paragraph{Step 1.}
For now, we fix the first reward parameter $\xb$ \emph{on the {premise} of} $\norm{\xb}_{\infty} \leq \alpha - \Delta$. Let $\bmu, \blambda \in \cV$ and consider two reward instances $(\xb, \bmu)$ and $(\xb,\blambda)$. From now, we omit the $\xb$ in the subscript and denote $(\xb,\bmu)$ by $\bmu$ to avoid notation clutter. Our first step is to prove that when $\eta\Delta$ is small enough, for any policy $\pi$, we have $\subopt_{\bmu}(\pi) + \subopt_{\blambda}(\pi) \gtrsim \eta\Delta^2 d_H(\bmu, \blambda)/K$ for all $\norm{\xb}_{\infty} \leq \alpha - \Delta$. 

We consider two instances, $\bmu_1$ and $\bmu_2$, correspondingly, such that $d_H(\bmu_1, \bmu_2) = m$, and denote the corresponding rewards by $r_1$ and $r_2$. Without loss of generality, we assume that $r_1$ and $r_2$ are given by
\begin{align*}
    & r_1(i) = 1/2 + x_i + \Delta, r_2(i) = 1/2 + x_i - \Delta, \ \forall i \in [1,l]; \\
    & r_1(i) = 1/2 + x_i - \Delta, r_2(i) = 1/2 + x_i +  \Delta, \ \forall i \in [l+1,m]; \\
    & r_1(i) = r_2(i) = 1/2 + r^*(i) + x_i, r^*(i) \in \{\pm \Delta\} , \ \forall i \in [m+1, K]; \\
    & r_1(i) = r_2(i) = 1/2 + \alpha , \ \forall i \in [K+1, 2K],
\end{align*}
where $0 \leq l \leq m$ and $m \leq K$ are some integers. Let $\pi^*_1$ and $\pi^*_2$ be the corresponding optimal policies under rewards $r_1$ and $r_2$. For simplicity, we use $\subopt_1(\pi)$ to denote $\subopt_{r_1}(\pi, \pi^*_1)$ and $\subopt_2(\pi)$ for $\subopt_{r_2}(\pi, \pi^*_2)$, correspondingly. By Lemma~\ref{lem:kl-subopt-decompose}, we know that
\begin{align*}
    \subopt_1(\pi) + \subopt_2(\pi) = \eta^{-1} \kl{\pi}{\pi^*_1} + \eta^{-1} \kl{\pi}{\pi^*_2}.
\end{align*}
Let $\pihat$ be the minimizer of the above equation, we know that $\pihat(i) \propto \sqrt{\pi^*_1(i)\pi^*_2(i)}$ and this gives
\begin{align*}
    & \subopt_1(\hat \pi) + \subopt_2(\hat \pi) \\
    & \quad = 2\eta^{-1} \log \frac{\sqrt{\sum_{i=1}^{2K} \exp(\eta r_1(i))}\sqrt{\sum_{j=1}^{2K} \exp(\eta r_2(j))}}{\sum_{k=1}^{2K} \exp\Big(\eta \big(r_1(k) + r_2(k)\big)/2\Big)} \\
    & \quad = \eta^{-1}\Bigg[\underbrace{\log \frac{\sum_{i=1}^{2K} \exp(\eta r_1(i))}{\sum_{k=1}^{2K} \exp\Big(\eta \big(r_1(k) + r_2(k)\big)/2\Big)}}_{X_1} + \underbrace{\log \frac{\sum_{i=1}^{2K} \exp(\eta r_2(i))}{\sum_{k=1}^{2K} \exp\Big(\eta \big(r_1(k) + r_2(k)\big)/2\Big)}}_{X_2}\Bigg].
\end{align*}
The first term $X_1$ can be computed as follows
\begin{align*}
    X_1 & = \log  \frac{\displaystyle \sum_{j=1}^l \exp(\eta x_j + \eta \Delta) + \sum_{j=l+1}^m \exp( \eta x_j - \eta \Delta) + \sum_{j=m+1}^K \exp\big(\eta x_j + \eta r^*(j)\big) + \sum_{j=K+1}^{2K} \exp(\eta\alpha)}{\displaystyle \sum_{j=1}^m \exp(\eta x_j) + \underbrace{\sum_{j=m+1}^K \exp\big(\eta x_j + \eta r^*(j)\big) + \sum_{j=K+1}^{2K} \exp(\eta\alpha)}_{M}} \\
    & = \log  \frac{\sum_{j=1}^l \exp(\eta x_j + \eta \Delta) + \sum_{j=l+1}^m \exp( \eta x_j - \eta \Delta) + M}{\sum_{j=1}^m \exp(\eta x_j) + M}.
\end{align*}
Similarly, we know that 
\begin{align*}
    X_2 = \log  \frac{\sum_{j=1}^l \exp(\eta x_j - \eta \Delta) + \sum_{j=l+1}^m \exp( \eta x_j + \eta \Delta) + M}{\sum_{j=1}^m \exp(\eta x_j) + M}.
\end{align*}
Now combining these two terms, we obtain that 
\begin{align}
    X_1 + X_2 & =\log  \frac{\sum_{j=1}^l \exp(\eta x_j + \eta \Delta) + \sum_{j=l+1}^m \exp( \eta x_j - \eta \Delta) + M}{\sum_{j=1}^m \exp(\eta x_j) + M} \notag \\
    & \quad + \log  \frac{\sum_{j=1}^l \exp(\eta x_j - \eta \Delta) + \sum_{j=l+1}^m \exp( \eta x_j + \eta \Delta) + M}{\sum_{j=1}^m \exp(\eta x_j) + M}. \label{eq:subopt-jensen}
\end{align}
Notice that
\begin{align*}
    & \sum_{j=1}^l \exp(\eta x_j + \eta \Delta) + \sum_{j=l+1}^m \exp( \eta x_j - \eta \Delta) + \sum_{j=1}^l \exp(\eta x_j - \eta \Delta) + \sum_{j=l+1}^m \exp( \eta x_j + \eta \Delta) \\
    & \quad = \sum_{j=1}^m \big(\exp(\eta x_j - \eta \Delta) + \exp(\eta x_j + \eta \Delta)\big),
\end{align*}
where the RHS is independent of $l$. Therefore, by the concavity of $x \mapsto \log x$, \Cref{eq:subopt-jensen} is minimized when the two terms differ the most, i.e., $l = 0$ or $l = m$. We thus obtain
\begin{align*}
    X_1 + X_2 & \geq \log  \frac{\sum_{j=1}^m \exp(\eta x_j + \eta \Delta) + M}{\sum_{j=1}^m \exp(\eta x_j) + M} + \log  \frac{\sum_{j=1}^m \exp(\eta x_j - \eta \Delta) + M}{\sum_{j=1}^m \exp(\eta x_j) + M} \\
    & = \log \frac{\big(\sum_{j=1}^m \exp(\eta x_j)\big)^2 + M^2 + M\sum_{j=1}^m \exp(\eta x_j) \big(\exp(\eta\Delta) + \exp(-\eta\Delta)\big)}{\big(\sum_{j=1}^m \exp(\eta x_j) + M\big)^2} \\
    & = \log \Bigg(1 + \frac{2M}{\big(\sum_{j=1}^m \exp(\eta x_j) + M\big)^2}\sum_{j=1}^m \bigg(\exp(\eta x_j) \bigg(\frac{\exp(\eta \Delta) + \exp(- \eta \Delta)}{2} - 1\bigg)\bigg) \Bigg).
\end{align*}
Now we come to bound the term $M$, which is straightforward since we have $-\alpha \leq  x_j \pm \Delta \leq \alpha$.
\begin{align*}
    m \leq K\exp(\eta\alpha) \leq M = \sum_{j=m+1}^K \exp\big(\eta x_j + \eta r^*(j)\big) + K \exp(\eta\alpha) \leq 2K\exp(\eta\alpha).
\end{align*}
Therefore, we know that
\begin{align*}
    \frac{2M}{\big(\sum_{j=1}^m \exp(\eta x_j) + M\big)^2}\sum_{j=1}^m \exp(\eta x_j) \geq \frac{2mM\exp(-\eta\alpha)}{9K^2\exp(2\eta\alpha)} \geq \frac{m}{5K\exp(2\eta\alpha)}.
\end{align*}
This enables us to bound the suboptimality gap as follows 
\begin{align}
    &\subopt_1(\hat \pi) + \subopt_2(\hat \pi) \notag\\
    & \quad  \geq \eta^{-1}\log \Bigg(1 + \frac{2M}{\big(\sum_{j=1}^m \exp(\eta x_j) + M\big)^2}\sum_{j=1}^m \bigg(\exp(\eta x_j) \bigg(\frac{\exp(\eta \Delta) + \exp(- \eta \Delta)}{2} - 1\bigg)\bigg) \Bigg) \notag\\
    & \quad \geq \eta^{-1} \log \Bigg(1 + \frac{m}{5K\exp(2\eta\alpha)} \cdot\bigg(\frac{\exp(\eta \Delta) + \exp(- \eta \Delta)}{2} - 1\bigg)\Bigg) \notag\\
    & \quad \geq \eta^{-1} \log \bigg(1 + \frac{m}{5K\exp(2\eta\alpha)}\eta^2\Delta^2\bigg), \label{eq:fastrate:lb:step1:penultimate}
\end{align}
where the last inequality holds due to $\forall x \in \RR, (\mathrm{e}^x + \mathrm{e}^{-x})/2 - 1 \geq x^2/2$. By $\alpha \geq 2\Delta$ and $ \max_{x\geq 0} x^2 - 5e^{4x} \leq 0$, we know that $m \eta^2\Delta^2 \leq 5K\exp(2\eta\alpha)$. Since $\forall x \in [0, 1], \log(1+x) \geq x/2$, we further have
\begin{align}
   \mathrm{\Cref{eq:fastrate:lb:step1:penultimate}} \geq \frac{m}{10K\exp(2\eta\alpha)}\eta\Delta^2, \label{eq:cond-of-separation}
\end{align}
which finishes our first step. 

\paragraph{Step 2.}
Let us first fix a time step $t \geq \eta^2 K$, and set $\alpha = 2\eta^{-1}\log 2$, which implies $\alpha\sqrt{t / K} \geq 1$, for all $t \geq \eta^2 K$, $\exists \Delta_t \in [0.5\sqrt{K/t}, \sqrt{K/t}]$ such that $\nicefrac{\alpha}{2\Delta_t} \in \NN^*$. Fixing such pair of $(t, \Delta_t)$ and setting $\Delta = \Delta_t$ in \cref{eq:cond-of-separation} yields that, for any policy $\pi$ and $\xb \in [-\alpha + \Delta_t, \alpha - \Delta_t]^K$,
\begin{align*}
    \EE_{\bmu \sim \unif(\cV)} \EE_{\bmu, t} \big[\subopt_{(\xb,\bmu),t}(\pi)\big] & \geq \frac{\eta\Delta_t^2}{2^{11}|\cV|K} \sum_{j=1}^K \sum_{\bmu \sim_j \blambda} \exp \big(-\KL(\PP_{\bmu,t}\|\PP_{\blambda,t})\big) \\
    & =  \frac{\eta\Delta_t^2}{2^{11}|\cV|K}\sum_{d_H(\bmu, \blambda)=1} \exp \big(-\KL(\PP_{\bmu,t}\|\PP_{\blambda,t})\big) \\
    & \geq \frac{\eta\Delta_t^2}{2^{11}} \exp \Bigg( - \frac{1}{|\cV|K} \sum_{d_H(\bmu, \blambda)=1}\KL(\PP_{\bmu,t}\|\PP_{\blambda,t}) \Bigg),
\end{align*}
where the first inequality is by plugging \cref{eq:cond-of-separation} (with $\exp(2\eta\alpha) = 16$) into \Cref{lem:assouad}, which gives the constant $1/(8 \cdot 160) \geq 2^{-11}$, and the last inequality holds due to Jensen's inequality, since there are $|\cV|K$ ordered pairs $(\bmu, \blambda)$ with $d_H(\bmu, \blambda)=1$.\footnote{The notation $\EE_{\bmu, t}[\cdot]$ is with respect to the trajectory distribution of the interaction between $\pi$ and the instance ${\bmu}$ up to time step $t$.} Then for any fixed $\bmu$, the standard divergence decomposition lemma~\citep[Lemma~15.1]{lattimore2020bandit} gives
\begin{align*}
    \sum_{\blambda: d_H(\bmu, \blambda)=1}\KL(\PP_{\bmu,t}\|\PP_{\blambda,t}) = \sum_{k=1}^K \EE_{\bmu, t}[N_t(k)] \KL(+\Delta_t, -\Delta_t) \leq 2t\Delta_t^2,
\end{align*}
where we recall that $\kl{+ \Delta_t}{-\Delta_t} = \kl{1/2 + x_j + \Delta_t}{1/2 + x_j -\Delta_t} = 2\Delta_t^2$ denotes the KL divergence from $\cN(1/2 + x_j + \Delta_t, 1)$ to $\cN(1/2 + x_j - \Delta_t, 1)$ and happens to be symmetric (\Cref{fact:kl-gaussian}).
Therefore, we know that 
\begin{align*}
    \EE_{\bmu \sim \unif(\cV)}  \EE_{\bmu, t}\big[\subopt_{(\xb,\bmu),t}(\pi)\big] & \geq \frac{\eta\Delta_t^2}{2^{11}} \exp \Bigg( - \frac{1}{|\cV|K} \sum_{d_H(\bmu, \blambda)=1}\KL(\PP_{\bmu,t}\|\PP_{\blambda,t}) \Bigg) \\
    & \geq \frac{\eta\Delta_t^2}{2^{11}} \exp \bigg(- \frac{2t\Delta_t^2}{K}\bigg) \\
    & \geq \frac{\eta K}{2^{13} \mathrm{e}^{2} t},
\end{align*}
where the last inequality holds due to $\Delta_t \in [\sqrt{K/t}/2, \sqrt{K/t}]$. Recall that $M_t \coloneqq \nicefrac{\alpha}{2\Delta_t}$ is a positive integer by design. We define $\cH_t \coloneqq \cup_{j=1}^{M_t} [-\alpha + (4j-3)\Delta_t, -\alpha + (4j- 1)\Delta_t]$, then we notice that if we take $\xb \sim \unif(\cH_t^K)$ and $\bmu \sim \unif(\cV)$ independently, then $\xb + \bmu\Delta_t \sim \unif([-\alpha, \alpha]^K)$. Therefore, the tower property gives
\begin{align}
    \EE_{(r_{[1:K]} - 1/2) \sim \unif([-\alpha, \alpha]^K)}  \EE_{\bmu, t}\big[\subopt_{r,t}(\pi)\big] & = \EE_{\xb \sim \unif(\cH_t^K)} \EE_{\bmu\sim \unif(\cV)}  \EE_{\bmu, t}\big[\subopt_{(\xb,\bmu),t}(\pi)\big] \notag \\ & \geq \frac{\eta K}{2^{13}\mathrm{e}^{2} t}, \label{eq:fast-lb:tower}
\end{align}
where $r_{[1:K]}$ denotes the first $K$ coordinates of the mean reward function $r_{\xb, {\bmu}}$ (See \Cref{fig:unif-prior} for an intuitive illustration of the equality in \cref{eq:fast-lb:tower}).
Invoking the tower property again yields that for any policy $\pi$,
\begin{align}
    \sup_{r} \EE_{(\pi, r)}\regret_r(T) &\geq \EE_{(r_{[1:K]} - 1/2) \sim \unif([-\alpha, \alpha]^K)}[\EE_{(\pi, r)}\regret_r(T)] \notag\\
    &\geq \sum_{t=\ceil{\eta^2K}}^{T} \EE_{(\xb, \bmu) \sim \unif(\cH_t^K \times \cV)} \EE_{\bmu, t}\big[ \subopt_{(\xb, \bmu), t}(\pi) \big] \geq \frac{\eta K}{2^{13}\mathrm{e}^{2}} \sum_{t=\ceil{\eta^2 K}}^T  t^{-1} \label{eq:fast-lb:penultimate},
\end{align}
where \cref{eq:fast-lb:penultimate} follows from \cref{eq:fast-lb:tower}. Finally, $\sum_{t=\ceil{\eta^2K}}^{T} t^{-1} = \Omega \big( \log(\nicefrac{T}{\eta^2K}) \big)$ concludes the proof.
\begin{figure}[h]
    \centering
    \includegraphics[width=0.618\linewidth]{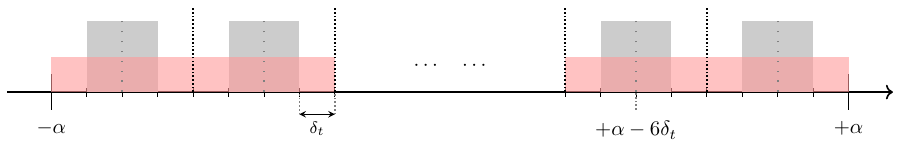}
    \caption{The shared uniform Bayes prior for every $t \geq \eta^2 K$. The plot above takes $1$ out of $K$ axes of $\unif\big([-\alpha, +\alpha]^K\big)$ for illustration. The \textbf{\textcolor{gray}{gray}} boxes denote the density of $\xb$ and hence the \textbf{\textcolor{red!40}{red}} boxes represent the density of $\xb + {\bmu}\Delta_t$.}
    \label{fig:unif-prior}
\end{figure}
\end{proof}

% \color{red}
\begin{remark}\label{rmk:sample-complexity-lower}
In general, a regret lower bound cannot be converted to a sample complexity lower bound by directly applying the online-to-batch conversion. Nevertheless, the proof of Theorem~\ref{thm:lowerbound-fast} indeed implies an $\Omega(\eta K /\epsilon)$ sample complexity lower bound. In particular, if we take $t=T$ in~\eqref{eq:fast-lb:tower} and then take maximum over all possible reward functions in the support of this Bayesian prior, then we obtain that, for any algorithm,
\begin{align*}
    \sup_r \EE_{(\pi, r)}\big[\subopt_{r,T}(\pi)\big] \gtrsim \frac{\eta K}{T}.
\end{align*}
This leads to the $\Omega(\eta K /\epsilon)$ sample complexity lower bound.
\end{remark}

\section{Auxiliary Lemmas}

We first recall a standard fact about the KL divergence between two Gaussian distributions with unit variance.
\begin{lemma}\label{fact:kl-gaussian}
    $\forall m,\Delta \in \RR, \KL({m}, {m + 2\Delta}) \coloneqq \kl{\cN(m, 1)}{\cN(m + 2\Delta, 1)} = 2\Delta^2$.
\end{lemma}

\begin{lemma}[Freedman's inequality, \citealt{freedman1975tail}]\label{lem:freedman-m}
    Let $M,v > 0$ be fixed constants. Let $\{x_i\}_{i=1}^{n}$ be a stochastic process, $\{\cF_{i}\}_i$ be a filtration so that for $i \in [n]$, $x_i$ is $\cF_i$-measurable, while almost surely
    \begin{align*}
        \EE[x_i | \cF_{i-1}] = 0, |x_i| \le M, \sum_{i=1}^n \EE\big[x_i^2|\cF_{i-1}\big] \le v.
    \end{align*}
    Then for any $\delta \in (0, 1)$, with probability at least $1-\delta$, we have
    \begin{align*}
        \sum_{i=1}^n x_i \le \sqrt{2v\log(1/\delta)} + 2/3 M\log(1/\delta).
    \end{align*}
\end{lemma}

\begin{lemma}[Azuma-Hoeffding inequality, \citealt{azuma1967weighted,cesa2006prediction}]\label{lem:azuma-hoeffding}
    Let $\{x_i\}_{i=1}^n$ be a martingale difference sequence with respect to a filtration $\{\cG_{i}\}$ satisfying $|x_i| \leq M$ for some constant $M$, $x_i$ is $\cG_{i+1}$-measurable, $\EE[x_i|\cG_i] = 0$. Then for any $0<\delta<1$, with probability at least $1-\delta$, we have 
    \begin{align}
        \sum_{i=1}^n x_i\leq M\sqrt{2n \log (1/\delta)}.\notag
    \end{align} 
\end{lemma}

\begin{lemma}[{\citealt[(D.10)]{zhao2026towards}}]\label{lem:kl-subopt-decompose}
Consider any $\eta > 0$, finite action set $\cA$, and reward function $r: \cA \to \RR$. Let $\piref \in \mathcal{P}(\cA)$ be any reference policy and $\pi^* \in \mathcal{P}(\cA)$ be the optimal policy under $r$, i.e., $\pi^*(a) \propto \piref(a)\exp(\eta r(a))$ for all $a \in \cA$. Let $\pi$ be any policy, then the suboptimality gap between $\pi$ and $\pi^*$ under the KL-regularized objective is given by $\subopt(\pi, \pi^*) = \eta^{-1} \kl{\pi}{\pi^*}$.
\end{lemma}

% \begin{lemma}[Freedman's inequality, \citealt{freedman1975tail}]\label{lem:freedman}
% Let Let $Z_0, Z_1, \cdots, Z_n$ be a martingale sequence of random variables such that for all $i$, there exists a constant $c_i$ such that $|Z_i - Z_{i-1}| < M$, and $\EE[(Z_i - Z_{i-1})^2] \leq \sigma_i^2$, then
% \begin{align*}
%     \PP[Z_n - Z_0 \geq t] \leq \exp\bigg(-\frac{t^2}{2\sum_{i=1}^n \sigma_i^2 + 2Mt/3}  \bigg).
% \end{align*}
% \end{lemma}

The following two lemmas are standard results for proving information-theoretic minimax lower bounds. 
\begin{lemma}[Le Cam's two-point method, \citealt{lecam1973convergence,yu1997assouad}]\label{lem:two-point}
    Let $\cR$ be the set of instances, $\Pi$ be the set of estimators, and $L:\Pi \times \cR \to \RR_+$ be a loss function. For $\tilde{r}, \bar{r} \in \cR$, suppose $\exists c > 0$ such that
\begin{align*}
    \inf_{\pi \in \Pi} \big[L(\pi, \tilde{r}) + L(\pi, \bar{r})\big] \geq c,
\end{align*}
then 
\begin{align*}
    \inf_{\pi \in \Pi} \sup_{r \in \cR} \EE_{\cD \sim P_r} L\big( \pi(\cD), r \big) \geq \frac{c}{4}\cdot \exp\big(-\kl{P_{\tilde{r}}}{P_{\bar{r}}} \big),
\end{align*}
where the trajectory distribution of $\pi$ interacting with instance $r$ is denoted by $P_r$.
\end{lemma}

We adopt the following variant of Assouad's lemma.\footnote{Similar variants have been shown in, e.g., \url{https://theinformaticists.wordpress.com/2019/09/16/lecture-8-multiple-hypothesis-testing-tree-fano-and-assoaud}}
\begin{lemma}[Assouad's Lemma, \citealt{yu1997assouad}]\label{lem:assouad}
   Let $\cR$ be the set of instances, $\Pi$ be the set of estimators, $\cV \coloneqq \{\pm1\}^S$ for some $S \in \NN^*$, such that $r_{\bnu} \in \cR$ for all ${\bnu} \in \cV$. Let $L: \Pi \times \cR \to \RR_+$ be any loss function satisfying the following separation condition 
\begin{align*}
    L(\pi, r_{\bmu}) + L(\pi, r_{\blambda}) \geq c\cdot d_H({\bmu}, {\blambda}), \ \forall {\bmu},{\blambda} \in \cV \text{ and } \pi \in \Pi
\end{align*}
for some $c \geq 0$, then for any estimator $\pi$,
\begin{align*}
     \EE_{{\bnu} \sim \unif(\cV)} \EE_{\cD \sim \PP_{\bnu}} L(\pi(\cD), r_{\bnu}) \geq \frac{c}{8|\cV|} \sum_{j=1}^S \sum_{
    \bmu\sim_j \blambda}\exp\big(-\KL(\PP_{\bmu}\|\PP_{\blambda})\big) ,
\end{align*}
where $\bmu \sim_j \blambda$ denotes that $d_H(\bmu, \blambda)=1$ and $\mu_j \neq \lambda_j$.
\end{lemma}
\begin{proof}[Proof of \Cref{lem:assouad}]

For any pair of policy $\pi$ and ${\bnu} \in \cV$, we pick their corresponding $\hat{\bnu} \in \argmin_{{\bnu}'\in \cV} L(\pi, r_{{\bnu}'})$ arbitrarily to obtain 
\begin{align*}
      L(\pi, r_{\bnu}) \geq \frac{L(\pi, r_{\bnu}) + L(\pi, r_{\hat{{\bnu}}})}{2} \geq \frac{c}{2} \sum_{j=1}^S  \bigg( \ind[\nu_j=1, \hat{\nu}_j = -1] + \ind[\nu_j=-1, \hat{\nu}_j = 1]\bigg), \forall {\bnu} \in \cV;
\end{align*}
which in turn implies
\begin{align*}
   \EE_{{\bnu} \sim \unif(\cV)}  L(\pi, r_{\bnu}) \geq \frac{c}{2} \sum_{j=1}^S \frac{1}{|\cV|}\bigg(\sum_{{\bnu}:\nu_j=1}\ind[\hat{\nu}_j = -1] + \sum_{{\bnu}:\nu_j=-1}\ind[\hat{\nu}_j = 1]\bigg).
\end{align*}

Then for any estimator $\pi$,
\begin{align*}
     \EE_{{\bnu} \sim \unif(\cV)} \EE_{\cD \sim \PP_{\bnu}} L(\pi(\cD), r_{\bnu}) &\geq \frac{c}{2} \sum_{j=1}^S \frac{1}{|\cV|}\bigg(\sum_{{\bnu}:\nu_j=1}\PP_{\bnu}[\hat{\nu}_j = -1] + \sum_{{\bnu}:\nu_j=-1}\PP_{\bnu}[\hat{\nu}_j = 1]\bigg) \\
     &= \frac{c}{2}\sum_{j=1}^{S} \frac{1}{2|\cV|} \sum_{\bmu \sim_j \blambda} \big( \PP_{\bmu}(\hat{\mu}_j = -1) + \PP_{\blambda}(\hat{\lambda}_j = +1) \big) \\
     &\geq \frac{c}{4|\cV|} \sum_{j=1}^{S} \sum_{\bmu \sim_j \blambda} \big(1 - \tv{\PP_{\bmu}}{\PP_{\blambda}}\big) \\
     &\geq \frac{c}{8|\cV|} \sum_{j=1}^S \sum_{
    \bmu\sim_j \blambda} \exp\big(-\KL(\PP_{\bmu}\|\PP_{\blambda})\big),
\end{align*}
where the penultimate inequality follows from the variational representation of $\mathsf{TV}$, and the last inequality is by the Bretagnolle-Huber inequality (See e.g., \citet[Theorem~14.2]{lattimore2020bandit}).
\end{proof}

\bibliographystyle{ims}
\bibliography{ref}

@inproceedings{xiong2024iterative,
  title={Iterative preference learning from human feedback: Bridging theory and practice for rlhf under kl-constraint},
  author={Xiong, Wei and Dong, Hanze and Ye, Chenlu and Wang, Ziqi and Zhong, Han and Ji, Heng and Jiang, Nan and Zhang, Tong},
  booktitle={Forty-first International Conference on Machine Learning},
  year={2024}
}

@inproceedings{
zhao2025sharp,
title={Sharp Analysis for {KL}-Regularized Contextual Bandits and {RLHF}},
author={Heyang Zhao and Chenlu Ye and Quanquan Gu and Tong Zhang},
booktitle={The Thirty-ninth Annual Conference on Neural Information Processing Systems},
year={2025}
}

@inproceedings{foster2025good,
  title={Is a Good Foundation Necessary for Efficient Reinforcement Learning? The Computational Role of the Base Model in Exploration},
  author={Foster, Dylan J and Mhammedi, Zakaria and Rohatgi, Dhruv},
  booktitle={The Thirty Eighth Annual Conference on Learning Theory},
  pages={2026--2142},
  year={2025},
  organization={PMLR}
}

@inproceedings{cai2020provably,
  title={Provably efficient exploration in policy optimization},
  author={Cai, Qi and Yang, Zhuoran and Jin, Chi and Wang, Zhaoran},
  booktitle={International Conference on Machine Learning},
  pages={1283--1294},
  year={2020},
  organization={PMLR}
}

@inproceedings{he2022near,
  title={Near-optimal policy optimization algorithms for learning adversarial linear mixture mdps},
  author={He, Jiafan and Zhou, Dongruo and Gu, Quanquan},
  booktitle={International Conference on Artificial Intelligence and Statistics},
  pages={4259--4280},
  year={2022},
  organization={PMLR}
}

@inproceedings{ji2024horizonfree,
  title={Horizon-free Reinforcement Learning in Adversarial Linear Mixture {MDP}s},
  author={Kaixuan Ji and Qingyue Zhao and Jiafan He and Weitong Zhang and Quanquan Gu},
  booktitle={The Twelfth International Conference on Learning Representations},
  year={2024}
}

@article{neu2017unified,
  title={A unified view of entropy-regularized markov decision processes},
  author={Neu, Gergely and Jonsson, Anders and G{\'o}mez, Vicen{\c{c}}},
  journal={arXiv preprint arXiv:1705.07798},
  year={2017}
}

@inproceedings{geist2019theory,
  title={A theory of regularized markov decision processes},
  author={Geist, Matthieu and Scherrer, Bruno and Pietquin, Olivier},
  booktitle={International Conference on Machine Learning},
  pages={2160--2169},
  year={2019},
  organization={PMLR}
}

@article{vieillard2020leverage,
  title={Leverage the average: an analysis of kl regularization in reinforcement learning},
  author={Vieillard, Nino and Kozuno, Tadashi and Scherrer, Bruno and Pietquin, Olivier and Munos, R{\'e}mi and Geist, Matthieu},
  journal={Advances in Neural Information Processing Systems},
  volume={33},
  pages={12163--12174},
  year={2020}
}

@article{kozuno2022kl,
  title={Kl-entropy-regularized rl with a generative model is minimax optimal},
  author={Kozuno, Tadashi and Yang, Wenhao and Vieillard, Nino and Kitamura, Toshinori and Tang, Yunhao and Mei, Jincheng and M{\'e}nard, Pierre and Azar, Mohammad Gheshlaghi and Valko, Michal and Munos, R{\'e}mi and others},
  journal={arXiv preprint arXiv:2205.14211},
  year={2022}
}

@article{schulman2017proximal,
  title={Proximal policy optimization algorithms},
  author={Schulman, John and Wolski, Filip and Dhariwal, Prafulla and Radford, Alec and Klimov, Oleg},
  journal={arXiv preprint arXiv:1707.06347},
  year={2017}
}

@inproceedings{haarnoja2018soft,
  title={Soft actor-critic: Off-policy maximum entropy deep reinforcement learning with a stochastic actor},
  author={Haarnoja, Tuomas and Zhou, Aurick and Abbeel, Pieter and Levine, Sergey},
  booktitle={International conference on machine learning},
  pages={1861--1870},
  year={2018},
  organization={PMLR}
}

@article{rafailov2023direct,
  title={Direct preference optimization: Your language model is secretly a reward model},
  author={Rafailov, Rafael and Sharma, Archit and Mitchell, Eric and Manning, Christopher D and Ermon, Stefano and Finn, Chelsea},
  journal={Advances in Neural Information Processing Systems},
  volume={36},
  year={2023}
}

@inproceedings{xie2025exploratory,
  title={Exploratory Preference Optimization: Harnessing Implicit Q*-Approximation for Sample-Efficient {RLHF}},
  author={Tengyang Xie and Dylan J Foster and Akshay Krishnamurthy and Corby Rosset and Ahmed Hassan Awadallah and Alexander Rakhlin},
  booktitle={The Thirteenth International Conference on Learning Representations},
  year={2025}
}

@article{richemond2024offline,
  title={Offline Regularised Reinforcement Learning for Large Language Models Alignment},
  author={Richemond, Pierre Harvey and Tang, Yunhao and Guo, Daniel and Calandriello, Daniele and Azar, Mohammad Gheshlaghi and Rafailov, Rafael and Pires, Bernardo Avila and Tarassov, Eugene and Spangher, Lucas and Ellsworth, Will and others},
  journal={arXiv preprint arXiv:2405.19107},
  year={2024}
}

@article{liu2024enhancing,
  title={Enhancing multi-step reasoning abilities of language models through direct q-function optimization},
  author={Liu, Guanlin and Ji, Kaixuan and Zheng, Renjie and Wu, Zheng and Dun, Chen and Gu, Quanquan and Yan, Lin},
  journal={arXiv preprint arXiv:2410.09302},
  year={2024}
}

@article{williams1992simple,
  title={Simple statistical gradient-following algorithms for connectionist reinforcement learning},
  author={Williams, Ronald J},
  journal={Machine learning},
  volume={8},
  pages={229--256},
  year={1992},
  publisher={Springer}
}

@inproceedings{ziebart2008maximum,
  title={Maximum entropy inverse reinforcement learning.},
  author={Ziebart, Brian D and Maas, Andrew L and Bagnell, J Andrew and Dey, Anind K and others},
  booktitle={Aaai},
  volume={8},
  pages={1433--1438},
  year={2008},
  organization={Chicago, IL, USA}
}

@inproceedings{levine2013guided,
  title={Guided policy search},
  author={Levine, Sergey and Koltun, Vladlen},
  booktitle={International conference on machine learning},
  pages={1--9},
  year={2013},
  organization={PMLR}
}

@inproceedings{ahmed2019understanding,
  title={Understanding the impact of entropy on policy optimization},
  author={Ahmed, Zafarali and Le Roux, Nicolas and Norouzi, Mohammad and Schuurmans, Dale},
  booktitle={International conference on machine learning},
  pages={151--160},
  year={2019},
  organization={PMLR}
}

@inproceedings{liu2021regularization,
  title={Regularization Matters in Policy Optimization - An Empirical Study on Continuous Control},
  author={Zhuang Liu and Xuanlin Li and Bingyi Kang and Trevor Darrell},
  booktitle={International Conference on Learning Representations},
  year={2021}
}

@article{ouyang2022training,
  title={Training language models to follow instructions with human feedback},
  author={Ouyang, Long and Wu, Jeffrey and Jiang, Xu and Almeida, Diogo and Wainwright, Carroll and Mishkin, Pamela and Zhang, Chong and Agarwal, Sandhini and Slama, Katarina and Ray, Alex and others},
  journal={Advances in neural information processing systems},
  volume={35},
  pages={27730--27744},
  year={2022}
}

@book{lattimore2020bandit,
  title={Bandit algorithms},
  author={Lattimore, Tor and Szepesv{\'a}ri, Csaba},
  year={2020},
  publisher={Cambridge University Press}
}

@book{zhang2023mathematical,
  title={Mathematical analysis of machine learning algorithms},
  author={Zhang, Tong},
  year={2023},
  publisher={Cambridge University Press}
}

@article{lecam1973convergence,
  title={Convergence of estimates under dimensionality restrictions},
  author={Le Cam, Lucien},
  journal={The Annals of Statistics},
  pages={38--53},
  year={1973},
  publisher={JSTOR}
}

@incollection{yu1997assouad,
  title={Assouad, fano, and le cam},
  author={Yu, Bin},
  booktitle={Festschrift for Lucien Le Cam: research papers in probability and statistics},
  pages={423--435},
  year={1997},
  publisher={Springer}
}

@inproceedings{zhao2025logarithmic,
 title={Logarithmic Regret for Online {KL}-Regularized Reinforcement Learning},
 author={Heyang Zhao and Chenlu Ye and Wei Xiong and Quanquan Gu and Tong Zhang},
 booktitle={Forty-second International Conference on Machine Learning},
 year={2025},
}

@inproceedings{zhao2026towards,
  title={Towards a Sharp Analysis of Offline Policy Learning for \$f\$-Divergence-Regularized Contextual Bandits},
  author={Qingyue Zhao and Kaixuan Ji and Heyang Zhao and Tong Zhang and Quanquan Gu},
  booktitle={The Fourteenth International Conference on Learning Representations},
  year={2026}
}

@article{azuma1967weighted,
  title={Weighted sums of certain dependent random variables},
  author={Azuma, Kazuoki},
  journal={Tohoku Mathematical Journal, Second Series},
  volume={19},
  number={3},
  pages={357--367},
  year={1967},
  publisher={Mathematical Institute, Tohoku University}
}

@article{freedman1975tail,
  title={On tail probabilities for martingales},
  author={Freedman, David A},
  journal={the Annals of Probability},
  pages={100--118},
  year={1975},
  publisher={JSTOR}
}

@inproceedings{wu2025greedy,
 title={Greedy Sampling Is Provably Efficient For {RLHF}},
 author={Di Wu and Chengshuai Shi and Jing Yang and Cong Shen},
 booktitle={The Thirty-ninth Annual Conference on Neural Information Processing Systems},
 year={2025}
}

@book{cesa2006prediction,
  title={Prediction, learning, and games},
  author={Cesa-Bianchi, Nicolo and Lugosi, G{\'a}bor},
  year={2006},
  publisher={Cambridge university press}
}

@article{abbasi2011improved,
  title={Improved algorithms for linear stochastic bandits},
  author={Abbasi-Yadkori, Yasin and P{\'a}l, D{\'a}vid and Szepesv{\'a}ri, Csaba},
  journal={Advances in neural information processing systems},
  volume={24},
  year={2011}
}

@article{russo2013eluder,
  title={Eluder dimension and the sample complexity of optimistic exploration},
  author={Russo, Daniel and Van Roy, Benjamin},
  journal={Advances in Neural Information Processing Systems},
  volume={26},
  year={2013}
}

@inproceedings{zhao2023optimal,
  title={Optimal online generalized linear regression with stochastic noise and its application to heteroscedastic bandits},
  author={Zhao, Heyang and Zhou, Dongruo and He, Jiafan and Gu, Quanquan},
  booktitle={International Conference on Machine Learning},
  pages={42259--42279},
  year={2023},
  organization={PMLR}
}

@article{orabona2019modern,
  title={A modern introduction to online learning},
  author={Orabona, Francesco},
  journal={arXiv preprint arXiv:1912.13213v8},
  year={2019}
}

@article{vovk2001competitive,
  title={Competitive on-line statistics},
  author={Vovk, Volodya},
  journal={International Statistical Review},
  volume={69},
  number={2},
  pages={213--248},
  year={2001},
  publisher={Wiley Online Library}
}

@article{singer2002universal,
  title={Universal linear least squares prediction: Upper and lower bounds},
  author={Singer, Andrew C and Kozat, Suleyman Serdar and Feder, Meir},
  journal={IEEE Transactions on Information Theory},
  volume={48},
  number={8},
  pages={2354--2362},
  year={2002},
  publisher={IEEE}
}

@article{guo2025deepseek,
  title={Deepseek-r1: Incentivizing reasoning capability in llms via reinforcement learning},
  author={Guo, Daya and Yang, Dejian and Zhang, Haowei and Song, Junxiao and Zhang, Ruoyu and Xu, Runxin and Zhu, Qihao and Ma, Shirong and Wang, Peiyi and Bi, Xiao and others},
  journal={arXiv preprint arXiv:2501.12948},
  year={2025}
}

@article{weng2025improved,
  title={Improved Bounds for Private and Robust Alignment},
  author={Weng, Wenqian and He, Yi and Zhou, Xingyu},
  journal={arXiv preprint arXiv:2512.23816},
  year={2025}
}

@article{wu2025offline,
  title={Offline and Online KL-Regularized RLHF under Differential Privacy},
  author={Wu, Yulian and Thareja, Rushil and Vepakomma, Praneeth and Orabona, Francesco},
  journal={arXiv preprint arXiv:2510.13512},
  year={2025}
}

@article{nayak2025achieving,
  title={Achieving Logarithmic Regret in KL-Regularized Zero-Sum Markov Games},
  author={Nayak, Anupam and Yang, Tong and Yagan, Osman and Joshi, Gauri and Chi, Yuejie},
  journal={arXiv preprint arXiv:2510.13060},
  year={2025}
}

@inproceedings{chu2011contextual,
  title={Contextual bandits with linear payoff functions},
  author={Chu, Wei and Li, Lihong and Reyzin, Lev and Schapire, Robert},
  booktitle={Proceedings of the fourteenth international conference on artificial intelligence and statistics},
  pages={208--214},
  year={2011},
  organization={JMLR Workshop and Conference Proceedings}
}

@article{robbins1952some,
  title={Some aspects of the sequential design of experiments},
  author={Robbins, Herbert},
  journal={Bulletin of the American Mathematical Society},
  volume={58},
  number={5},
  pages={527--535},
  year={1952}
}

@article{lai1985asymptotically,
  title={Asymptotically efficient adaptive allocation rules},
  author={Lai, Tze Leung and Robbins, Herbert},
  journal={Advances in applied mathematics},
  volume={6},
  number={1},
  pages={4--22},
  year={1985},
  publisher={Academic Press, Inc. Orlando, FL, USA}
}

@article{agrawal1995sample,
  title={Sample mean based index policies by o (log n) regret for the multi-armed bandit problem},
  author={Agrawal, Rajeev},
  journal={Advances in applied probability},
  volume={27},
  number={4},
  pages={1054--1078},
  year={1995},
  publisher={Cambridge University Press}
}

@article{auer2002finite,
  title={Finite-time analysis of the multiarmed bandit problem},
  author={Auer, Peter and Cesa-Bianchi, Nicolo and Fischer, Paul},
  journal={Machine learning},
  volume={47},
  number={2},
  pages={235--256},
  year={2002},
  publisher={Springer}
}

@inproceedings{audibert2009minimax,
  title={Minimax policies for adversarial and stochastic bandits},
  author={Audibert, Jean-Yves and Bubeck, S{\'e}bastien},
  booktitle={COLT},
  pages={217--226},
  year={2009}
}

@article{lai1987adaptive,
  title={Adaptive treatment allocation and the multi-armed bandit problem},
  author={Lai, Tze Leung},
  journal={The annals of statistics},
  pages={1091--1114},
  year={1987},
  publisher={JSTOR}
}

@article{burnetas1996optimal,
  title={Optimal adaptive policies for sequential allocation problems},
  author={Burnetas, Apostolos N and Katehakis, Michael N},
  journal={Advances in Applied Mathematics},
  volume={17},
  number={2},
  pages={122--142},
  year={1996},
  publisher={Elsevier}
}

@article{auer2002nonstochastic,
  title={The nonstochastic multiarmed bandit problem},
  author={Auer, Peter and Cesa-Bianchi, Nicolo and Freund, Yoav and Schapire, Robert E},
  journal={SIAM journal on computing},
  volume={32},
  number={1},
  pages={48--77},
  year={2002},
  publisher={SIAM}
}

@article{lattimore2018refining,
  title={Refining the confidence level for optimistic bandit strategies},
  author={Lattimore, Tor},
  journal={Journal of Machine Learning Research},
  volume={19},
  number={20},
  pages={1--32},
  year={2018}
}

@inproceedings{degenne2016anytime,
  title={Anytime optimal algorithms in stochastic multi-armed bandits},
  author={Degenne, R{\'e}my and Perchet, Vianney},
  booktitle={International Conference on Machine Learning},
  pages={1587--1595},
  year={2016},
  organization={PMLR}
}

@inproceedings{tiapkin2023fast,
  title={Fast rates for maximum entropy exploration},
  author={Tiapkin, Daniil and Belomestny, Denis and Calandriello, Daniele and Moulines, Eric and Munos, Remi and Naumov, Alexey and Perrault, Pierre and Tang, Yunhao and Valko, Michal and Menard, Pierre},
  booktitle={International Conference on Machine Learning},
  pages={34161--34221},
  year={2023},
  organization={PMLR}
}

@inproceedings{chen2013combinatorial,
  title={Combinatorial multi-armed bandit: General framework and applications},
  author={Chen, Wei and Wang, Yajun and Yuan, Yang},
  booktitle={International conference on machine learning},
  pages={151--159},
  year={2013},
  organization={PMLR}
}

@inproceedings{zhang2024settling,
  title={Settling the sample complexity of online reinforcement learning},
  author={Zhang, Zihan and Chen, Yuxin and Lee, Jason D and Du, Simon S},
  booktitle={The Thirty Seventh Annual Conference on Learning Theory},
  pages={5213--5219},
  year={2024},
  organization={PMLR}
}

@article{zhou2022computationally,
  title={Computationally efficient horizon-free reinforcement learning for linear mixture mdps},
  author={Zhou, Dongruo and Gu, Quanquan},
  journal={Advances in neural information processing systems},
  volume={35},
  pages={36337--36349},
  year={2022}
}

@inproceedings{kleinberg2008multi,
  title={Multi-armed bandits in metric spaces},
  author={Kleinberg, Robert and Slivkins, Aleksandrs and Upfal, Eli},
  booktitle={Proceedings of the fortieth annual ACM symposium on Theory of computing},
  pages={681--690},
  year={2008}
}

@inproceedings{steck2018calibrated,
  title={Calibrated recommendations},
  author={Steck, Harald},
  booktitle={Proceedings of the 12th ACM conference on recommender systems},
  pages={154--162},
  year={2018}
}

@inproceedings{geyik2019fairness,
  title={Fairness-aware ranking in search \& recommendation systems with application to linkedin talent search},
  author={Geyik, Sahin Cem and Ambler, Stuart and Kenthapadi, Krishnaram},
  booktitle={Proceedings of the 25th acm sigkdd international conference on knowledge discovery \& data mining},
  pages={2221--2231},
  year={2019}
}

\end{document}